\documentclass{article}

\usepackage[T1]{fontenc}
\usepackage[htt]{hyphenat}

\usepackage{microtype}
\usepackage{graphicx}
\usepackage{subfigure}
\usepackage{booktabs} % for professional tables
\usepackage{natbib}
\usepackage{algorithm}
\usepackage{algorithmic}
\usepackage{amsmath}
\usepackage{amssymb}
\usepackage{graphicx}
\usepackage{mathtools}
\usepackage{amsfonts}
\usepackage{amsthm,bm}
\usepackage{color}
\usepackage{enumitem}
\usepackage{dsfont}
\usepackage{pgfplots}
\pgfplotsset{width=10cm,compat=1.9}
 \usepackage{pifont}%http://ctan.org/pkg/pifont
\usepackage{thmtools} 
\usepackage{thm-restate}
\usepackage{sidecap}

\newcommand{\R}{\mathbb{R}}

\renewcommand{\phi}{\varphi}

\graphicspath {{figures/}}

\usepackage{hyperref}
\usepackage{caption}
\usepackage[super]{nth}
\newtheorem{lemma}{Lemma}

\newtheorem{theorem}{Theorem}

\newtheorem{proposition}{Proposition}

\mathtoolsset{showonlyrefs}

\usepackage{breakcites}
\usepackage{authblk}
\usepackage[margin=1.38in]{geometry}
\usepackage{enumitem}

\usepackage{hyperref}

\hypersetup{
     colorlinks = true,
     linkcolor = brown,
     anchorcolor = blue,
     citecolor = teal,
     filecolor = blue,
     urlcolor = black
     }

\title{Depth and Feature Learning are Provably Beneficial for Neural Network Discriminators}

\author[a]{Carles Domingo-Enrich}
\affil[a]{Courant Institute of Mathematical Sciences, New York University}
\begin{document}

\maketitle

%\tableofcontents

\begin{abstract}%
  We construct pairs of distributions $\mu_d, \nu_d$ on $\mathbb{R}^d$ such that the quantity $|\mathbb{E}_{x \sim \mu_d} [F(x)] - \mathbb{E}_{x \sim \nu_d} [F(x)]|$ decreases as $\Omega(1/d^2)$ for some three-layer ReLU network $F$ with polynomial width and weights, while declining exponentially in $d$ if $F$ is any two-layer network with polynomial weights. This shows that deep GAN discriminators are able to distinguish distributions that shallow discriminators cannot. Analogously, we build pairs of distributions $\mu_d, \nu_d$ on $\mathbb{R}^d$ such that $|\mathbb{E}_{x \sim \mu_d} [F(x)] - \mathbb{E}_{x \sim \nu_d} [F(x)]|$ decreases as $\Omega(1/(d\log d))$ for two-layer ReLU networks with polynomial weights, while declining exponentially for bounded-norm functions in the associated RKHS. This confirms that feature learning is beneficial for discriminators. Our bounds are based on Fourier transforms.
\end{abstract}

\section{Introduction} \label{sec:intro}
Wasserstein generative adversarial networks (WGANs, \cite{arjovsky2017wasserstein}) are a well-known generative modeling technique where synthetic samples are generated as $x = g(z)$, where $g : \R^{d_0} \to \R^d$ is known as the \textit{generator} and $z$ is a sample from a $d_0$-dimensional standard Gaussian random variable. In order to make the generated distribution close to the data samples available, the generator is a neural network trained by minimizing the loss $\max_{f} \mathbb{E}_{x \sim p_{\text{data}}}[f(x)] - \mathbb{E}_{z \sim \mathcal{N}(0,\text{Id})}[f(x)]$, where the function $f : \R^d \to \R$ is the \textit{discriminator} and it is also a neural network. Both the generator and the discriminator are typically deep networks (i.e. depth larger than two) with architectures that are tailored to the task at hand. Given our loose understanding of the optimization of deep networks and our better grasp of two-layer networks, a natural question to ask is the following: \textit{do deep discriminators offer any provable advantages over shallow ones?} This is the issue that we tackle in this paper; namely, we showcase distributions that are easily distinguishable by three-layer ReLU discriminators but not by two-layer ones.

The study of theoretical separation results between two-layer and three-layer networks began with the works of \cite{martens2013onthe} and \cite{eldan2016thepower}. The two papers show pairs of a function $f : \R^d \to \R$ and a distribution $\mathcal{D}$ on $\R^d$ such that $f$ can be approximated with respect to $\mathcal{D}$ by a three-layer network of widths polynomial in $d$, but not by any polynomial-width two-layer networks. That is, \cite{eldan2016thepower} show that for any $g$ is expressed as a two-layer network of width at most $c e^{cd}$ for some universal constant $c > 0$, then $E_{x \sim \mathcal{D}} (f(x) - g(x))^2 > c$. \cite{daniely2017depth} shows a simpler setting where the exponential dependency is improved to $d \log(d)$ and the non-approximation results extend to networks with polynomial weight magnitude. \cite{safran2017depth} provide other examples where similar behavior holds, \cite{telgarsky2016benefits} gives separation results beyond depth 3, and \cite{venturi2021depth} generalize the work of \cite{eldan2016thepower}. Note that all the results in these works concern function approximations in the $L^2(\mathcal{D})$ norm.

Our work establishes separation results between two-layer and three-layer networks of a similar flavor, for the task of discriminating distributions on high-dimensional Euclidean spaces. Our main result (\autoref{sec:separation_3_v_2}) can be summarised in the following theorem:
\begin{theorem} [Informal] \label{thm:informal_1}
For any $d \in \mathbb{Z}^{+}$, there exist probability measures $\mu_d, \nu_d \in \mathcal{P}(\R^d)$ and a three-layer network $F$ of widths $O(d)$ and weight magnitude $1$ such that $|\mathbb{E}_{x \sim \mu_d} [F(x)] - \mathbb{E}_{x \sim \nu_d} [F(x)]| = \Omega(1/d^2)$, but such that for any two-layer network $G$ of weight magnitude $O(1)$, $|\mathbb{E}_{x \sim \mu_d} [G(x)] - \mathbb{E}_{x \sim \nu_d} [G(x)]| = O(d^2 \kappa^d)$, where $\kappa = 0.7698\dots$
\end{theorem}
That is, there exists a three-layer network $F$ with polynomial widths and weights such that the difference of expectations of $F$ with respect to $\mu_d$ and $\nu_d$ decreases only quadratically with $d$, but for all such two-layer networks, the difference of expectations decreases exponentially. We formalize the vague notion weight magnitude as a specific path-norm of the weights, but the choice of the weight norm does not alter the essence of the result. Unlike the separation result of \cite{eldan2016thepower}, which relies on radial functions, we build $\mu_d$ and $\nu_d$ using parity functions and some additional tricks.

Our second contribution (\autoref{sec:separation_2_v_rkhs}) is to provide analogous separation results between two-layer neural networks and functions in the unit ball of the associated reproducing kernel Hilbert space (RKHS) $\mathcal{H}$ (see \autoref{sec:framework}). While two-layer networks are \textit{feature-learning}, functions in $\mathcal{H}$ are \textit{lazy}; they can be seen intuitively as infinitely wide two-layer networks for which the first layer features are sampled i.i.d from a fixed distribution. Our result is as follows:
\begin{theorem} [Informal] \label{thm:informal_2}
    For any $d \in \mathbb{Z}^{+}$, there exist probability measures $\mu_d, \nu_d \in \mathcal{P}(\R^d)$ and a two-layer network $F$ of weight magnitude $1$ such that $|\mathbb{E}_{x \sim \mu_d} [F(x)] - \mathbb{E}_{x \sim \nu_d} [F(x)]| = \Omega(\frac{1}{d \log(d)})$, but such that for any $G \in \mathcal{H}$ with $\|G\|_{\mathcal{H}} \leq 1$, $|\mathbb{E}_{x \sim \mu_d} [G(x)] - \mathbb{E}_{x \sim \nu_d} [G(x)]| = O(d \exp(-\frac{(\sqrt{d}-1)^2}{16}))$.
\end{theorem}
The recent work \cite{domingoenrich2021separation} provides similar results for probability measures $\mu_d, \nu_d$ on the hypersphere $\mathbb{S}^{d-1}$ such their difference of densities is proportional to a spherical harmonic of order proportional to $d$, and they leave open the extension of the separation result to densities on $\R^d$ with only high-frequency differences. Our theorem solves the issue, as our measures $\mu_d, \nu_d$ have density difference proportional to $\sin(\ell \langle x, e_1 \rangle)$ times a Gaussian density, where the frequency $\ell$ increases as $\sqrt{d}$.
Experimentally, the superiority of feature-learning over fixed-kernel discriminators has been observed for the CIFAR-10 and MNIST datasets \citep{li2017mmd, santos2017learning}.

\vspace{-5pt}
\section{Framework} \label{sec:framework}
\paragraph{Notation.} $\mathbb{S}^{d-1}$ denotes the $(d-1)$-dimensional hypersphere (as a submanifold of $\R^d$). %and $\mathcal{B}_R(\R^d)$ is the Euclidean open ball of radius $R$.
For $U \subseteq \R^d$ measurable, %the space $C_0(U)$ of functions vanishing at infinity contains the continuous functions $f$ such that for any $\epsilon > 0$, there exists compact $K \subseteq U$ depending on $f$ such that $|f(x)| < \epsilon$ for $x \in U \setminus K$. 
$\mathcal{P}(U)$ is the set of Borel probability measures, 
$\mathcal{M}(U)$ is the space of finite signed Radon measures (Radon measures for shortness). $(x)_{+}$ denotes $\max\{x,0\}$. %(which may be seen as the dual of $C_0(U)$).
%Throughout the paper, the term Radon measure refers to a finite signed Radon measure for shortness. %If $\gamma \in \mathcal{M}(U)$, then ${\|\gamma\|}_{\text{TV}}$ is the total variation (TV) norm of $\gamma$. $\mathcal{M}_{\mathbb{C}}(U)$ denotes the space of complex-valued finite signed Radon measures, defined as the dual space of $C_0(U,\mathbb{C})$ (the space of complex-valued functions vanishing at infinity).

\vspace{-6pt}
\paragraph{Schwartz functions and tempered distributions.} We denote by $\mathcal{S}(\R^d)$ the space of Schwartz functions, which contains the functions $\phi$ in $\mathcal{C}^{\infty}(\R^{d})$ whose derivatives of any order decay faster than polynomials of all orders, i.e. for all $k, r \in (\mathbb{N}_0)^d$, $p_{k,r}(\phi) = \sup_{x \in \R^d} |x^{k} \partial^{(r)} \phi(x)| < +\infty$. We denote by $\mathcal{S}'(\R^{d})$ the dual space of $\mathcal{S}(\R^d)$, which is known as the space of tempered distributions on $\R^{d}$. Tempered distributions $T$ can be characterized as linear mappings $\mathcal{S}(\R^d) \to \R$ such that given ${(\phi_m)}_{m \geq 0} \subseteq \mathcal{S}(\R^d)$, if $\lim_{m \to \infty} p_{k,r}(\phi_m) = 0$ for any $k,r \in (\mathbb{Z}^{+})^2$, then $\lim_{m \to \infty} T(\phi_m) = 0$. 
Functions that grow no faster than polynomials can be embedded in $\mathcal{S}'(\R^{d})$ by defining $\langle g, \phi \rangle := \int_{\R^{d}} \varphi(x) g(x) \ dx$ for any $\varphi \in \mathcal{S}(\R^{d})$.
%We denote by $\mathbb{P}(\R^d)$ the space of polynomials over $\R^d$, and for any $K \subseteq \R^d$, $\mathbb{P}(K)$ denotes the space of functions on $\R^d$ which are restrictions of polynomials on $K$ and zero on $\R^d \setminus K$. 

\vspace{-6pt}
\paragraph{Fourier transforms.} For $f \in L^1(\R^{d})$, we use $\hat{f}$ to denote the unitary Fourier transform with angular frequency, defined as $\hat{f}(\xi) = \frac{1}{(2\pi)^{d/2}} \int_{\R^{d}} f(x) e^{-i \langle \xi, x \rangle} dx$, and the inverse Fourier transfom as $\check{f}(\xi) = \frac{1}{(2\pi)^{d/2}} \int_{\R^{d}} f(x) e^{-i \langle \xi, x \rangle} dx$. If $\hat{f} \in L^1(\R^{d})$ as well, we have the inversion formula $f(x) = \check{\hat{f}}(x)$. %$f(x) = \frac{1}{(2\pi)^{d/2}} \int_{\R^d} \hat{f}(\xi) e^{i \langle \xi, x \rangle} dx$. 
The Fourier transform is a continuous automorphism on $\mathcal{S}(\R^d)$, and it is defined for a tempered distribution $T \in \mathcal{S}'(\R^d)$ as $\hat{T} \in \mathcal{S}'(\R^d)$ fulfilling $\langle \hat{T}, \phi \rangle = \langle T, \hat{\phi} \rangle$. 

\vspace{-6pt}
\paragraph{Convolutions.} If $f \in \mathcal{S}'(\R^d), g \in \mathcal{S}(\R^d)$ the convolution of $f$ and $g$ is defined as the tempered distribution $f*g \in \mathcal{S}'(\R^d)$ such that for any Schwartz test function $\phi \in \mathcal{S}(\R^d)$, $\langle f*g, \phi \rangle = \langle g(y), \langle \phi(x + y), f(x) \rangle \rangle$.
%If $f \in \mathcal{S}'(\R^d), g \in \mathcal{S}(\R^d)$ then $f*g \in \mathcal{S}'(\R^d)$ and 
Moreover, it turns out that $f*g \in \mathcal{S}(\R^d)$, and we have that $\widehat{f * g} = (2\pi)^{d/2} \hat{f} \hat{g}$ (\cite{strichartz2003guide}, Sec. 4.3), a result known as the convolution theorem. Note that the factor $(2\pi)^{d/2}$ is specific to the unitary, angular-frequency Fourier transform.
%The following well-known theorem states that $\widehat{f*g}$ is indeed a tempered distribution and it relates relates the Fourier transform $\widehat{f*g}$ to $\hat{f}$ and $\hat{g}$.
%The Fourier transform may be defined for any tempered distribution $T \in \mathcal{S}'(\R^{d})$ as the tempered distribution $\hat{T}$ that acts on $\varphi \in \mathcal{S}(\R^{d})$ as $\langle \hat{T}, \varphi \rangle = \langle T, \hat{\varphi} \rangle$.

\vspace{-6pt}
\paragraph{Neural networks and path-norms.} A generic three-layer neural network $f : \R^d \to \R$ with activation function $\sigma : \R \to \R$ and weights $\mathcal{W} = ({(\theta_j,b_j)}_{j=1:m_1},{(W_{i,j})}_{i=1:m_2,j=0:m_1},{(w_i)}_{i=1:m_2})$ can be written as
\begin{align} \label{eq:three_layer_nn}
    f_{\mathcal{W}}(x) = \sum_{i=1}^{m_2} w_i \sigma \left(\sum_{j=1}^{m_1} W_{i,j} \sigma\left(\langle \theta_j, x \rangle - b_j \right) + W_{i,0} \right) + w_0.
\end{align}
There are several ways of measuring the magnitude of the weights of a neural network \citep{neyshabur2017exploring, neyshabur2018apac, bartlett2017spectrally}. The classical view is that a particular weight norm is useful if it gives rise to tight generalization bounds for the class of neural networks with bounded norm (although the work \cite{nagajaran2019uniform} shows that this approach may be unable to provide a complete picture of generalization). For the sake of convenience, in our work we make use of the following path-norms with and without bias\footnote{\cite{neyshabur2017exploring} studies the $l^1$ and $l^2$ path-norms. Note that our choice is the $l^1$ path-norm, but using the $l^2$ norm for the first-layer weights, %. The advantage of this choice is that it 
which defaults to the $\mathcal{F}_1$ norm introduced by \cite{bach2017breaking} for two-layer networks.}:
\begin{align}
\begin{split} \label{eq:l2_path_norms}
    &\text{PN}_{b}(\mathcal{W}) = \sum_{i=1}^{m_2} |w_i| \bigg( \sum_{j=1}^{m_1} |W_{i,j}| \cdot \|(\theta_j,b_j)\|_2 + |W_{i,0}| \bigg) + |w_0|, \\ &\text{and} \quad \text{PN}_{nb}(\mathcal{W}) = \sum_{i=1}^{m_2} |w_i| \bigg( \sum_{j=1}^{m_1} |W_{i,j}| \cdot \|\theta_j\|_2 \bigg)
    %\text{$l^2$-PN}_{b}(W) = \bigg( &\sum_{i=1}^{m_2} |w_2^{(i)}|^2 \bigg( \sum_{j=1}^{m_1} |w_1^{(i,j)}|^2 \bigg( \sum_{k=1}^{d} |w_0^{(j,k)}|^2 + |w_0^{(j,0)}|^2 \bigg) + |w_1^{(i,0)}|^2 \bigg) + |w_2^{(0)}|^2 \bigg)^{1/2}, \\
    %&\text{and} \quad \text{$l^2$-PN}_{nb}(W) = \bigg( \sum_{i=1}^{m_2} |w_2^{(i)}|^2 \bigg( \sum_{j=1}^{m_1} |w_1^{(i,j)}|^2 \bigg( \sum_{k=1}^{d} |w_0^{(j,k)}|^2 \bigg) \bigg) \bigg)^{1/2},
\end{split}
\end{align}
respectively. %Two-layer neural networks can be written in an expression analogous to \eqref{eq:three_layer_nn} (note that any two-layer network can be written as a three-layer network), but for convenience and shortness we will write them in the following notation: 
Similarly, two-layer neural networks can be written as
\begin{align} \label{eq:two_layer_nn}
    f_{\mathcal{W}} = \sum_{i=1}^{m} w_i \sigma(\langle \theta_i, x \rangle - b_i) + w_0, \quad \text{where } \mathcal{W} = (w^{(i)},\theta_i,b_i)_{i=0:m},
\end{align}
and the path-norms read $\text{PN}_{b}(\mathcal{W}) = \sum_{i=1}^{m} |w_i| \cdot \|(\theta_i,b_i)\|_2 + |w_0|$, $\text{PN}_{nb}(\mathcal{W}) = \sum_{i=1}^{m} |w_i| \cdot \|(\theta_i,b_i)\|_2$.
%and their $l^2$ path-norms admit formulas analogous to \eqref{eq:l2_path_norms}. In fact, note that any two-layer network can be written as a three-layer network. 

\vspace{-6pt}
\paragraph{RKHS associated to two-layer neural networks.} We define $\mathcal{H}$ as the RKHS of functions $\R^d \to \R$ associated the kernel $k(x,y) = \int_{\mathbb{S}^{d-1}\times \R} \sigma(\langle \theta, x \rangle - b) \sigma(\langle \theta, y \rangle - b) \, d\tau(\theta,b)$,
where $\tau \in \mathcal{P}(\mathbb{S}^{d-1}\times \R)$ is an arbitrary fixed probability measure. In our paper we will use $\tau = \text{Unif}(\mathbb{S}^{d-1}) \otimes \mathcal{N}(0,1)$, but previous papers have tudied and given closed forms for slightly different kernels \citep{leroux2007continuous,cho2009kernel}. Functions in the space $\mathcal{H}$ may be written as \citep{bach2017breaking}
\begin{align} \label{eq:f_rkhs}
    f_h(x) = \int_{\mathbb{S}^{d-1}\times \R} \sigma(\langle \theta, x \rangle - b) h(\theta,b) \, d\tau(\theta,b), \quad \text{where } h \in L^2(\tau).
\end{align}
The RKHS norm of a function $f \in \mathcal{H}$ may be written as  $\|f\|_{\mathcal{H}}^2 = \inf\{ \|h\|^2_{L^2(\tau)} \, | \, \forall x \in \R^d, \, f(x) = f_h(x) \}$, where $\|h\|^2_{L^2(\tau)} = \int_{\mathbb{S}^{d-1}\times \R} h(\theta,b)^2 \, d\tau(\theta,b)$. The characterization \eqref{eq:f_rkhs} showcases the connection of $\mathcal{H}$ with neural networks; if we were two replace $h(\theta,b) \, d\tau(\theta,b)$ by a Radon measure of the form $\sum_{i=1}^{m} w^{(i)} \delta_{(\theta_i,b_i)}$, we would obtain a two-layer network. It turns out that in general, two-layer networks do not belong to $\mathcal{H}$ and can only be approximated by functions with an exponential RKHS norm \citep{bach2017breaking}. %In theory and experiments, spaces 

\vspace{-6pt}
\paragraph{Integral probability metrics.} Integral probability metrics (IPM) are pseudometrics on $\mathcal{P}(\R^d)$ of the form $d_{\mathcal{F}}(\mu,\nu)= \sup_{f \in \mathcal{F}} |\mathbb{E}_{x \sim \mu}f(x) - \mathbb{E}_{x \sim \nu} f(x)|,$
where $\mathcal{F}$ is a class of functions from $\R^d$ to $\R$. IPMs may be regarded as an abstraction of WGAN discriminators; the class $\mathcal{F}$ can encode a specific network architecture and parameter constraints or regularization. In this paper, we study IPMs with the following three choices for $\mathcal{F}$:
\begin{itemize}
    \item $\mathcal{F}_{3L}$ is the class of ReLU (or leaky ReLU) three-layer networks $f_{W}$ of the form \eqref{eq:three_layer_nn} with bounded path-norm with bias: $\text{PN}_{b}(\mathcal{W}) \leq 1$. Upon simplification, the IPM takes the form
    \begin{align} \label{eq:3L_ipm}
    d_{\mathcal{F}_{3L}}(\mu,\nu) = \sup_{\sum_{j=1}^{m_1} |w_j| \cdot \|(\theta_j,b_j)\|_2 + |w_0| \leq 1} \left|\int_{\R^d} \sigma \left(\sum_{j=1}^{m_1} w_j \sigma(\langle \theta_j, x \rangle - b_j) + w_0 \right) \, d(\mu-\nu)(x) \right|.
    \end{align}
    \item $\mathcal{F}_{2L}$ is the class of two-layer ReLU networks $f_{\mathcal{W}}$ of the form \eqref{eq:two_layer_nn} with bounded path-norm without bias: $\text{PN}_{b}(\mathcal{W}) \leq 1$. The IPM takes the form
    \begin{align} \label{eq:2L_ipm}
    d_{\mathcal{F}_{2L}}(\mu,\nu) = \sup_{(\theta,b) \in \mathbb{S}^{d-1} \times \R} \left|\int_{\R^d} \sigma \left(\langle \theta, x \rangle - b \right) \, d(\mu-\nu)(x) \right|.
    \end{align}
    \item $\mathcal{F}_{\mathcal{H}}$ is the class of functions in the RKHS $\mathcal{H}$ with RKHS norm less or equal than 1 (setting $\sigma$ as the ReLU or leaky ReLU). Upon simplification, the IPM takes the form
    \begin{align} \label{eq:H_ipm}
    d_{\mathcal{F}_{\mathcal{H}}}(\mu,\nu) = \left(\int_{\mathbb{S}^{d-1} \times \R} \left(\int_{\R^d} \sigma \left(\langle \theta, x \rangle - b \right) \, d(\mu-\nu)(x) \right)^2 \, d\tau(\theta,b) \right)^{1/2}.
    \end{align}
    IPMs for RKHS balls are known as maximum mean discrepancies (MMD), introduced by \cite{gretton2007kernel,gretton2012akernel}. They admit an alternative closed form in terms of the kernel $k$. Just like neural network IPMs give rise to GANs, if we use the MMD instead, we obtain a related generative modeling technique: generative moment matching networks (GMMNs, \cite{li2015generative,dziugaite2015training}).
\end{itemize}
Note that the neural networks in \eqref{eq:3L_ipm}, \eqref{eq:2L_ipm} are simpler than the respective generic form of three-layer and two-layer networks; in fact, the last layers have just one neuron with weight 1 and no bias terms. The reason behind this is that convex functions on convex sets attain their minima at extreme points. \autoref{sec:ipm_derivations} provides brief derivations of the expressions \eqref{eq:3L_ipm}, \eqref{eq:2L_ipm}, and a pointer to the proof of \eqref{eq:H_ipm}.

\vspace{-5pt}
\section{Separation between three-layer and two-layer discriminators} \label{sec:separation_3_v_2}
\paragraph{The pair $(\mu_d,\nu_d)$.} Let $\sigma > 0$ and define the set $\mathcal{B} = \{-\frac{3}{2}, -\frac{1}{2}, \frac{1}{2}, \frac{3}{2}\} \subseteq \R$, and the sets $\mathcal{B}^d_{+} = \{ x \in \mathcal{B}^d \ | \prod_{i=1}^{d} x_i > 0\}$, $\mathcal{B}^d_{-} = \{ x \in \mathcal{B}^d \ | \prod_{i=1}^{d} x_i < 0\}$. Define the probability measures $\mu_d, \nu_d \in \mathcal{P}(\R^d)$ with densities $\frac{d\mu_d}{dx} = \rho^{+}_d, \frac{d\nu_d}{dx} = \rho^{-}_d$ defined as
\begin{align}
    \rho^{+}_d(x) = \frac{2}{(4\sqrt{2\pi \sigma^2})^{d}} \sum_{\beta \in \mathcal{B}^d_{+}} \exp(-\frac{\|x-\beta\|^2}{2\sigma^2}), \quad \rho^{-}_d(x) = \frac{2}{(4\sqrt{2\pi \sigma^2})^{d}} \sum_{\beta \in \mathcal{B}^d_{-}} \exp(-\frac{\|x-\beta\|^2}{2\sigma^2}).
\end{align}
Remark that $\rho^{+}_d$ and $\rho^{-}_d$ are normalized because $|\mathcal{B}^d_{+}| = |\mathcal{B}^d_{-}| = \frac{4^d}{2}$. %If $\xi^+, \xi^{-}$ are random vectors distributed uniformly over $\mathcal{B}^d_{+}$ and $\mathcal{B}^d_{-}$ respectively, and $X$ is a $d$-variate Gaussian $\mathcal{N}(0, \sigma^2 \text{Id})$, the variables $Z^{+} = \xi^+ + X$ and $Z^{-} = \xi^- + X$ are distributed according to $\mu_d$ and $\nu_d$ respectively.
The Radon measure $\mu_d - \nu_d$ has density
\begin{align}
    \rho_d(x) := \rho^{+}_d(x) - \rho^{-}_d(x) = \frac{2}{(4\sqrt{2\pi \sigma^2})^{d}} \sum_{\beta \in \mathcal{B}^d} \prod_{i=1}^{d} \chi_{\beta_i} \exp(-\frac{\|x-\beta\|^2}{2\sigma^2}),
\end{align}
where we use the short-hand $\chi_{\beta_i} = \text{sign}(\beta_i)$. 
%\input{colt_2021/rho_d_tikz}
%\textcolor{red}{In the end, change to rho_d_tikz}
\begin{figure}[t]
\begin{center}
\begin{tikzpicture}[scale=0.46]
\begin{axis}[
    xmin = -2.5, xmax = 2.5,
    ymin = -2.5, ymax = 2.5]
    \addplot[
        domain = -2.5:2.5,
        samples = 200,
        smooth,
        thick,
        blue,
    ] {(exp(-(x-1.5)^2/(2*0.1^2)) + exp(-(x-0.5)^2/(2*0.1^2)) - exp(-(x+1.5)^2/(2*0.1^2)) - exp(-(x+0.5)^2/(2*0.1^2)))/(3*0.1*sqrt(2*pi))};
    %{exp(-x/10)*( cos(deg(x)) + sin(deg(x))/10 )};
\end{axis}
\end{tikzpicture}
\quad
\includegraphics[scale=0.14]{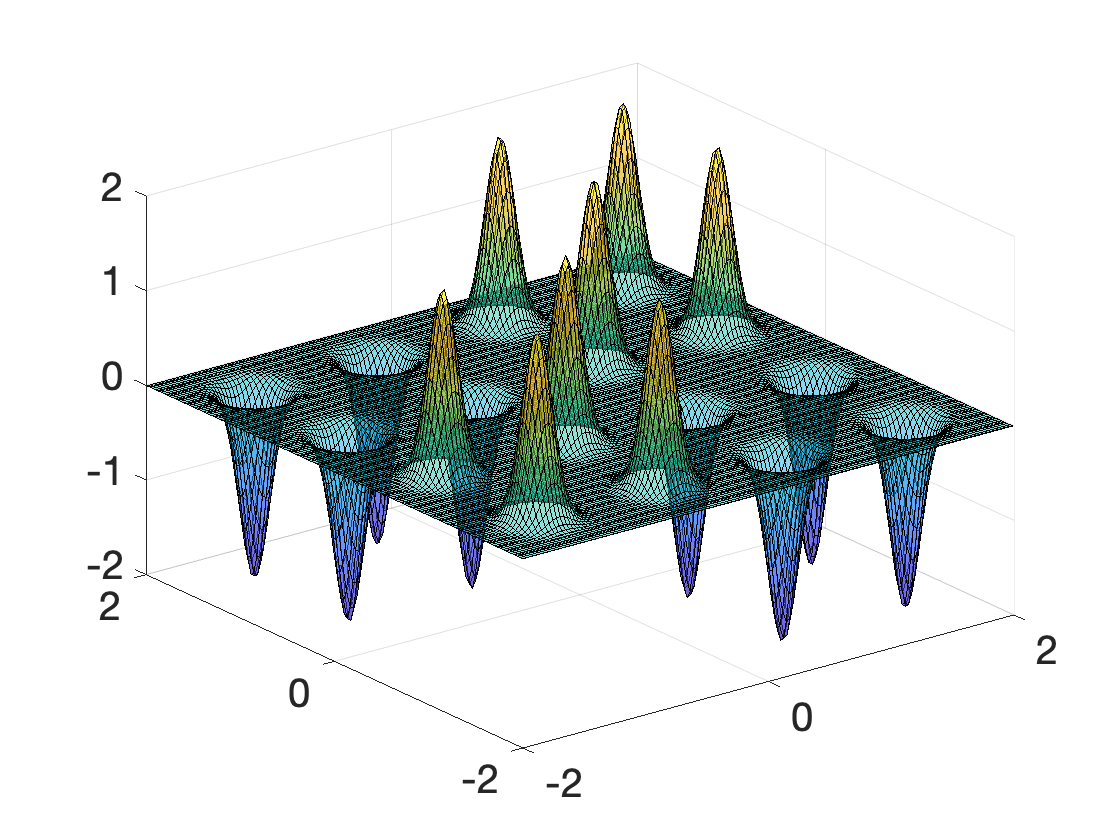}
\quad
\begin{tikzpicture}[scale=0.46]
\begin{axis}[ 
axis x line=center, axis y line = center, xmin=-2.5, xmax=2.5, ymin=-1.5, ymax=1.5]
%[ xlabel=$x$, ylabel=$y$, axis x line=center, axis y line = center, xmin=-3.5, xmax=3.5, ymin=-1.5, ymax=1.5]
\addplot+[ycomb,mark=triangle,mark options={rotate=180}] plot coordinates {(-1.5,-0.5) (-0.5,-0.5)};
\addplot+[ycomb,mark=triangle,mark options={rotate=0}] plot coordinates {(1.5,0.5) (0.5,0.5)};
%\draw[->] (axis cs:3,0) -- (axis cs:3,1);
\end{axis}
\end{tikzpicture}
\end{center}
\caption{Left: Plot of the density $\rho_{d}$ for $d=1$ with $\sigma = 0.1$. Center: Plot of the density $\rho_{d}$ for $d=2$ with $\sigma = 0.1$. Right: Plot of the measure of $\pi_{d}$ for $d=1$. Arrows denote Dirac delta functions; their length and sign denote the signed mass allocated at each position.}
\label{fig:rho_density_sin_2_3}
\end{figure}
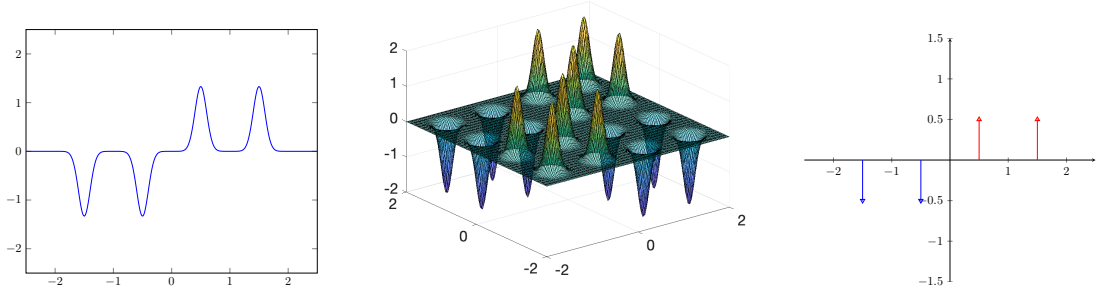

\subsection{Upper bound for two-layer discriminators}
In this subsection we provide an upper bound on the two-layer IPM $d_{\mathcal{F}_{2L}}(\mu_d,\nu_d)$ that decreases exponentially with the dimension $d$, via a Fourier-based argument.

\vspace{-6pt}
\paragraph{The Fourier transform of $\rho_d$.} Let $\pi_d = \frac{2}{4^d} \sum_{\beta \in \mathcal{B}^d} \prod_{i=1}^{d} \chi_{\beta_i} \delta_{\beta} = 2 \bigotimes_{i=1}^{d} (\frac{1}{4} \sum_{\beta_i \in \mathcal{B}} \chi_{\beta_i} \delta_{\beta_i})$, where $\delta_x$ denotes the Dirac delta at the point $x$. Formally, $\pi_d$ is a tempered distribution. Let $g_d$ be the density of the $d$-variate Gaussian $\mathcal{N}(0, \sigma^2 \text{Id})$. The following lemma, proved in \autoref{sec:proofs_3_v_2}, writes the density $\rho_d$ in terms of $\pi_d$ and $g_d$.
\begin{lemma} \label{lem:rho_d_conv}
We can write $\rho_d$ as a convolution of the tempered distribution $\pi_d$ with the Schwartz function $g_d$. That is, $\rho_d = \pi_d * g_d$.
\end{lemma}
Thus, we have that $\widehat{\rho_d} = \widehat{\pi_d * g_d} = (2\pi)^{d/2} \widehat{\pi_d} \cdot \widehat{g_d}$. %It is known \cite{erdelyi1954tables,kammler2000afirst} that the unitary, angular Fourier transform of $x \mapsto e^{-\alpha x^2}$ is $\omega \mapsto \frac{1}{\sqrt{2\alpha}} e^{-\frac{\omega^2}{4\alpha}}$. Thus, since $g_d(x) = \frac{1}{(2\pi \sigma^2)^{d/2}} e^{-\frac{\|y\|^2}{2\sigma^2}}$, we can take $\alpha = \frac{1}{2 \sigma^2}$, and obtain that $\widehat{g_d}(\omega) = \frac{1}{(2\pi \sigma^2)^{d/2}} \sigma^d e^{-\frac{\sigma^2 \|\omega\|^2}{2}} = \frac{1}{(2\pi)^{d/2}} e^{-\frac{\sigma^2 \|\omega\|^2}{2}}$. 
It is known \citep{erdelyi1954tables,kammler2000afirst} that the (unitary, angular-frequency) Fourier transform of $g_d(x) = \frac{1}{(2\pi \sigma^2)^{d/2}} e^{-\frac{\|y\|^2}{2\sigma^2}}$ is $\widehat{g_d}(\omega) = \frac{1}{(2\pi)^{d/2}} e^{-\frac{\sigma^2 \|\omega\|^2}{2}}$. Also, since the Fourier transform of $x \mapsto \sin(k x)$ is $\omega \mapsto \sqrt{2\pi} \frac{\delta(\omega-k)-\delta(\omega+k)}{2i}$, we have that the Fourier transform of $x \mapsto \frac{\delta(x-k)-\delta(x+k)}{4}$ is $\omega \mapsto %\frac{i}{2\sqrt{2\pi}} \sin(-k \omega) =
-\frac{i}{2\sqrt{2\pi}} \sin(k \omega)$. Thus,
\begin{align}
    \widehat{\pi_d}(\omega) &= 2 \prod_{i=1}^{d} \bigg(\widehat{\frac{1}{4} \sum_{\beta_i \in \mathcal{B}} \chi_{\beta_i} \delta_{\beta_i}} \bigg)(\omega_i) = 2 \prod_{i=1}^{d} \left( \frac{-i}{2\sqrt{2\pi}} \left(\sin \left(\frac{\omega_i}{2} \right) + \sin \left(\frac{3\omega_i}{2} \right) \right) \right) \\ &= 2 \left( \frac{-i}{\sqrt{2\pi}} \right)^d \prod_{i=1}^{d} \cos \left( \omega_i \right) \sin \left( 2 \omega_i \right),
\end{align}
where the last equality follows from the identity $\sin(\alpha) + \sin(\beta) = 2 \sin(\frac{\alpha+\beta}{2}) \cos(\frac{\alpha-\beta}{2})$. Consequently,
\begin{align}
    \widehat{\rho_d}(\omega) = 2 \left( \frac{-i}{\sqrt{2\pi}} \right)^d \prod_{i=1}^{d} e^{-\frac{\sigma^2 \omega_i^2}{2}} \cos \left( \omega_i \right) \sin \left( 2 \omega_i \right).
\end{align}

\vspace{-6pt}
\paragraph{Expressing $\mathbb{E}_{x \sim \mu_d}[\sigma(\langle \theta, x \rangle - b)] - \mathbb{E}_{x \sim \nu_d}[\sigma(\langle \theta, x \rangle - b)]$ in terms of $\widehat{\rho_d}$.} Note that $\mathbb{E}_{x \sim \mu_d}[\sigma(\langle \theta, x \rangle - b)] - \mathbb{E}_{x \sim \nu_d}[\sigma(\langle \theta, x \rangle - b)]$ is equal to $\int_{\R^d} \sigma(\langle \theta, x \rangle - b) \rho_d(x) \, dx$, for any $(\omega,b) \in \mathbb{S}^{d-1} \times \R$. The following proposition, which is proved in \autoref{sec:proofs_3_v_2} and based on Lemma 3 of \cite{domingoenrich2021separation}, may be used to reexpress this in terms of $\widehat{\rho_d}$.
\begin{proposition} \label{prop:inner_prod_fourier}
Take $(\theta,b) \in \mathbb{S}^{d-1} \times \R$ arbitrary. For any $\phi \in \mathcal{S}(\R^d)$ and any activation $\phi : \R \to \R$ belonging to the space of tempered distributions $\mathcal{S}(\R)$. Then, we have
\begin{align}
    \int_{\R^d} \phi(x) \sigma(\langle \theta, x \rangle - b) \, dx = (2\pi)^{(d-1)/2} \langle \hat{\sigma}(t), \hat{\phi}(-t\theta) e^{-itb} \rangle.
\end{align}
\end{proposition}
An application of \autoref{prop:inner_prod_fourier} yields $\int_{\R^d} \rho_d(x) \sigma(\langle \theta, x \rangle - b) \, dx = (2\pi)^{(d-1)/2} \langle \hat{\sigma}(t), \widehat{\rho_d}(-t\theta) e^{-itb} \rangle.$ Note that 
\begin{align} \label{eq:fourier_rho_t_theta}
    (2\pi)^{(d-1)/2} \widehat{\rho_d}(-t\theta) e^{-itb} = -\sqrt{\frac{2}{\pi}} \left( -i \right)^d e^{-\frac{\sigma^2 t^2}{2} + itb} \prod_{i=1}^d \cos(t \theta_i) \sin(2 t \theta_i)
\end{align}
The following lemma provides the expressions of the Fourier transforms $\hat{\sigma}$ of the ReLU and leaky ReLU activations, as tempered distributions on $\R$. 
\begin{lemma} [\cite{domingoenrich2021separation}, App. B] \label{lem:relu_fourier}
Take $\sigma : \R \to \R$ of the form $\sigma(x) = c_{+}(x)_{+}^{\alpha} + c_{-}(-x)_{+}^{\alpha}$, where $c_{+}, c_{-} \in \R$ and $\alpha \in \mathbb{Z}^{+}$. For $\alpha = 1$, $c_{+} = 1$, $c_{-} = 0$ corresponds to the ReLU, and $c_{+} = 1$, $c_{-} \in (-1,0)$ corresponds to the leaky ReLU. Then,
\begin{align}
    \hat{\sigma}(\omega) = A \frac{d^{\alpha}}{d\omega^{\alpha}} \left( \text{p.v.} \left[ \frac{1}{i\pi\omega} \right] \right) + B \frac{d^{\alpha}}{d\omega^{\alpha}} \delta(\omega),
\end{align}
where $A = i^{\alpha-1} \frac{\alpha!}{\sqrt{2\pi}} (c_{+} - (-1)^{\alpha} c_{-})$ and $B = i^{\alpha} \sqrt{\frac{\pi}{2}} (c_{+} - (-1)^{\alpha} c_{-}) + (-i)^{\alpha} c_{-}$.
\end{lemma}
Here $\text{p.v.} \left[ \frac{1}{\omega} \right]$ is a Cauchy principal value, defined as $\text{p.v.} \left[ \frac{1}{\omega} \right] (\phi) = \lim_{\epsilon \to 0} \int_{\R \setminus [-\epsilon,\epsilon]} \frac{1}{
\omega} \phi(\omega) \, d\omega = \int_{0}^{+\infty} \frac{\phi(\omega) - \phi(-\omega)}{\omega}$.
Moreover, the derivative of a tempered distribution $f \in \mathcal{S}'(\R)$ is defined in the weak sense: $\langle \frac{df}{d\omega}, \phi \rangle = -\langle f, \frac{d\phi}{d\omega} \rangle$.
Applying \autoref{lem:relu_fourier} with $\alpha = 1$ on equation \eqref{eq:fourier_rho_t_theta}, we have that 
\begin{align} \label{eq:discriminator_fourier_0}
    &\int_{\R^d} \rho_d(x) \sigma(\langle \theta, x \rangle - b) \, dx = (2\pi)^{(d-1)/2} \langle \hat{\sigma}(t), \widehat{\rho_d}(-t\theta) e^{-itb} \rangle \\ &= -\sqrt{\frac{2}{\pi}} \left( -i \right)^d \int_{\R} \left( A \frac{d}{dt} \left( \text{p.v.} \left[ \frac{1}{i\pi t} \right] \right) + B \frac{d}{dt} \delta(t) \right) \left( e^{-\frac{\sigma^2 t^2}{2} - itb} \prod_{i=1}^d \cos(t \theta_i) \sin(2 t \theta_i) \right) \, dt 
\end{align}
We can compute this explicitly. First,
\begin{align}
    \int_{\R} \frac{d}{dt} \delta(t) \left( e^{-\frac{\sigma^2 t^2}{2} - itb} \prod_{i=1}^d \cos(t \theta_i) \sin(2 t \theta_i) \right) \, dt = -\frac{d}{dt} \left( e^{-\frac{\sigma^2 t^2}{2} - itb} \prod_{i=1}^d \cos(t \theta_i) \sin(2 t \theta_i) \right) \bigg\rvert_{t=0} = 0,
\end{align}
which holds because the factors $\sin(2 t \theta_i)$ are equal to 0 when $t = 0$. Second,
\begin{align}
\begin{split} \label{eq:pv_computation}
    &\int_{\R} \frac{d}{dt} \left(\text{p.v.} \left[ \frac{1}{i\pi t} \right] \right) \left( e^{-\frac{\sigma^2 t^2}{2} - itb} \prod_{i=1}^d \cos(t \theta_i) \sin(2 t \theta_i) \right) \, dt \\ &= - \text{p.v.} \left[ \frac{1}{i\pi t} \right] \left(\frac{d}{dt} \left( e^{-\frac{\sigma^2 t^2}{2} - itb} \prod_{i=1}^d \cos(t \theta_i) \sin(2 t \theta_i) \right) \right)
\end{split}
\end{align}
The following lemma, proved in \autoref{sec:proofs_3_v_2}, provides an upper bound strategy for Cauchy principal values:
\begin{lemma} \label{lem:pv_upper_bound}
For any $\delta > 0$, $|\text{p.v.}[\frac{1}{x}](u)| \leq 2 \left( \sup_{x \in (-1,1)} |u'(x)| + \frac{1}{\delta}\sup_{x \in \R \setminus [-1,1]} |u(x) \cdot x^{\delta}| \right)$.
\end{lemma}
Let us set
\begin{align}
\begin{split} \label{eq:u_def}
    u(t) &= \frac{d}{dt} \left( e^{-\frac{\sigma^2 t^2}{2} - itb} \prod_{i=1}^d \cos(t \theta_i) \sin(2 t \theta_i) \right) \\ &= \left(-\sigma^2 t - ib - \sum_{i=1}^d \frac{\theta_i \sin(t \theta_i)}{\cos(t \theta_i)} + 2 \sum_{i=1}^d \frac{\theta_i \cos(2t \theta_i)}{\sin(2t \theta_i)} \right) e^{-\frac{\sigma^2 t^2}{2} - itb} \prod_{i=1}^d \cos(t \theta_i) \sin(2 t \theta_i). 
\end{split}
\end{align}
For ease of computation, in the last equality we introduced some removable singularities. \autoref{lem:sup_bound_u_u_prime} in \autoref{sec:proofs_3_v_2} provides the following bounds:
\begin{align} \label{eq:sup_bound_u_u_prime}
\sup_{x \in \R} |u'(x)| \leq O\left( \kappa^d \left( d^2 + d |b| + b^2 \right) \right), \text{ and } \sup_{x \in \R} |u(x) \cdot x^{\delta}| \leq O\left( \kappa^d \left(\frac{d + |b|}{\sigma} \right) \right).
\end{align}
The key idea of the proof of \autoref{lem:sup_bound_u_u_prime} (and of the whole construction in this section) is the inequality $\sup_{t \in \R} |\prod_{i=1}^d \cos(t \theta_i) \sin(2 t \theta_i)| \leq \kappa^d$, where $\kappa := \sup_{x \in \R} |\cos(t) \sin(2 t)| = 0.7698\dots$ (see \autoref{fig:sin_cos}). Since $\kappa < 1$, the factor $|\prod_{i=1}^d \cos(t \theta_i) \sin(2 t \theta_i)|$ is exponentially small in the dimension $d$.

Plugging the bounds \eqref{eq:sup_bound_u_u_prime} into \autoref{lem:pv_upper_bound} yields an upper bound on the absolute value of \eqref{eq:pv_computation}. In consequence, the following upper bound holds: 
\begin{proposition} \label{prop:upper_bound_two_layers}
We have $\left| \int_{\R^d} \rho_d(x) \sigma(\langle \theta, x \rangle - b) \, dx \right| \leq O\left( \kappa^d \left( d^2 + d |b| + b^2 + \frac{d + |b|}{\sigma} \right) \right)$ for any $(\theta,b) \in \mathbb{S}^{d-1} \times \R$.
\end{proposition}

\vspace{-6pt}
\paragraph{Concluding the upper bound.} \autoref{prop:upper_bound_two_layers} shows that if $|b| \leq d+\sqrt{d}$, then we can write $\left| \int_{\R^d} \rho_d(x) \sigma(\langle \theta, x \rangle - b) \, dx \right| \leq O\left( \kappa^d \left( d^2 + \frac{d}{\sigma} \right) \right).$ That is, unless $|b|$ is large, $\left| \int_{\R^d} \rho_d(x) \sigma(\langle \theta, x \rangle - b) \, dx \right|$ decreases exponentially with the dimension $d$.
In the following, we show that for large $d$, this is also the case. Namely,
\begin{lemma} \label{lem:b_d_sqrtd}
If $|b| > d+\sqrt{d}$, then $\left| \int_{\R^d} \rho_d(x) \sigma(\langle \theta, x \rangle - b) \, dx \right| \leq \frac{\sigma}{\sqrt{2\pi}} e^{-\frac{d^2}{2\sigma^2}}$.
\end{lemma}
\autoref{lem:b_d_sqrtd}, which is proved in \autoref{sec:proofs_3_v_2}, allows us to conclude the upper bound.
\begin{theorem} \label{thm:upper_bound_2l}
The following inequality holds for the IPM between $\mu_d$ and $\nu_d$ corresponding to the class $\mathcal{F}_{2L}$ of two-layer networks:
\begin{align}
    d_{\mathcal{F}_{2L}}(\mu_d,\nu_d) = \sup_{(\theta,b) \in \mathbb{S}^{d-1}} \left| \int_{\R^d} \rho_d(x) \sigma(\langle \theta, x \rangle - b) \, dx \right| \leq O\left( \max\left\{ \kappa^d \left( d^2 + \frac{d}{\sigma} \right), \sigma e^{-\frac{d^2}{2\sigma^2}} \right\}\right).
\end{align}
\end{theorem}

\subsection{Lower bound for three-layer discriminators.}
In order to provide a lower bound on the IPM $d_{\mathcal{F}_{3L}}(\mu_d,\nu_d)$ we construct a specific three-layer network $F$, and then show a lower bound on $|\mathbb{E}_{x \sim \mu_d}[F(x)] - \mathbb{E}_{x \sim \nu_d}[F(x)]|$ and an upper bound on the path-norm of $F$.

\vspace{-6pt}
\paragraph{Construction of the discriminator $F$.} Let us fix $0 < x_0 < 1/4$ arbitrary. Define the two-layer network $f_1 : \R \to \R$ as
\begin{align} \label{eq:f1_exp}
    f_1(x) = \sum_{\beta \in \mathcal{B}} \frac{\text{sign}(\beta)}{x_0} \left( (x-(\beta-2x_0))_{+} - (x-(\beta-x_0))_{+} - (\beta-(b+x_0))_{+} + (x-(\beta+2x_0))_{+} \right)
\end{align}
The function $f_1$, which is plotted in \autoref{fig:f_1_2} (left), takes non-zero values only around points in $\mathcal{B}$, and it takes value 1 around positive $\beta \in \mathcal{B}$, and value -1 around negative $\beta \in \mathcal{B}$.

If $d$ is even, we define the two-layer network
$f_2 : \R \to \R$ as
\begin{align} 
\begin{split} \label{eq:f_2_def_even}
    f_2(x) &= 1 - (x)_{+} - (-x)_{+} -(-1)^{d/2} ((x-d)_{+} +(-x-d)_{+}) \\ &- 2 \sum_{i=1}^{(d-2)/2} (-1)^i ((x-2i)_{+} +(-x-2i)_{+}).
\end{split}
\end{align}
This function is plotted for $d=4$ in \autoref{fig:f_1_2} (center), and it takes alternating values $\pm 1$ at even integers.
If $d \geq 3$ is odd, we define $f_2$ as
\begin{align} 
\begin{split} \label{eq:f_2_def_odd}
    f_2(x) &= x + (-1)^{(d-1)/2} (-(x-d)_{+} + (-x-d)_{+}) \\ &+ 2 \sum_{i=0}^{(d-3)/2} (-1)^i (-(x-2i-1)_{+} + (-x-2i-1)_{+}).
\end{split}
\end{align}
This function is plotted for $d=5$ in \autoref{fig:f_1_2} (right), and it takes alternating values $\pm 1$ at odd integers.
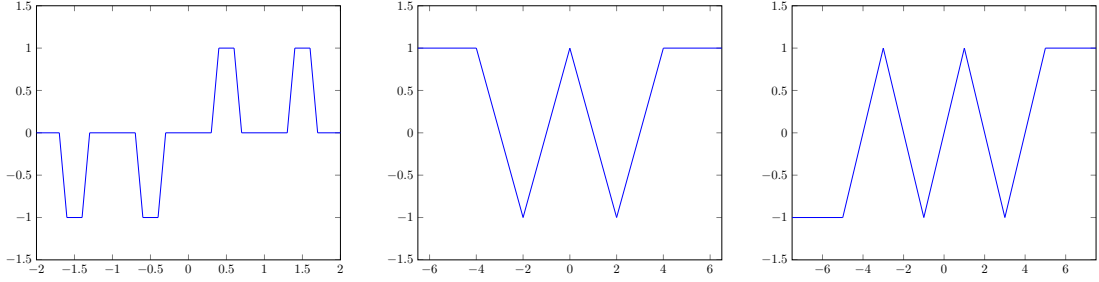
\begin{figure}[t]
\begin{center}
\begin{tikzpicture}[scale=0.48]
\begin{axis}[
    xmin = -2, xmax = 2,
    ymin = -1.5, ymax = 1.5]
    \addplot[
        domain = -2:-1.7,
        samples = 10,
        smooth,
        thick,
        blue,
    ] {0};
    \addplot[
        domain = -1.7:-1.6,
        samples = 10,
        smooth,
        thick,
        blue,
    ] {-(x+1.7)/0.1};
    \addplot[
        domain = -1.6:-1.4,
        samples = 10,
        smooth,
        thick,
        blue,
    ] {-1};
    \addplot[
        domain = -1.4:-1.3,
        samples = 10,
        smooth,
        thick,
        blue,
    ] {-1+(x+1.4)/0.1};
    \addplot[
        domain = -1.3:-0.7,
        samples = 10,
        smooth,
        thick,
        blue,
    ] {0};
    \addplot[
        domain = -0.7:-0.6,
        samples = 10,
        smooth,
        thick,
        blue,
    ] {-(x+0.7)/0.1};
    \addplot[
        domain = -0.6:-0.4,
        samples = 10,
        smooth,
        thick,
        blue,
    ] {-1};
    \addplot[
        domain = -0.4:-0.3,
        samples = 10,
        smooth,
        thick,
        blue,
    ] {-1+(x+0.4)/0.1};
    \addplot[
        domain = -0.3:0.3,
        samples = 10,
        smooth,
        thick,
        blue,
    ] {0};
    \addplot[
        domain = 0.3:0.4,
        samples = 10,
        smooth,
        thick,
        blue,
    ] {(x-0.3)/0.1};
    \addplot[
        domain = 0.4:0.6,
        samples = 10,
        smooth,
        thick,
        blue,
    ] {1};
    \addplot[
        domain = 0.6:0.7,
        samples = 10,
        smooth,
        thick,
        blue,
    ] {1-(x-0.6)/0.1};
    \addplot[
        domain = 0.7:1.3,
        samples = 10,
        smooth,
        thick,
        blue,
    ] {0};
    \addplot[
        domain = 1.3:1.4,
        samples = 10,
        smooth,
        thick,
        blue,
    ] {(x-1.3)/0.1};
    \addplot[
        domain = 1.4:1.6,
        samples = 10,
        smooth,
        thick,
        blue,
    ] {1};
    \addplot[
        domain = 1.6:1.7,
        samples = 10,
        smooth,
        thick,
        blue,
    ] {1-(x-1.6)/0.1};
    \addplot[
        domain = 1.7:2,
        samples = 10,
        smooth,
        thick,
        blue,
    ] {0};
    %{exp(-x/10)*( cos(deg(x)) + sin(deg(x))/10 )};
\end{axis}
\end{tikzpicture}
\quad
\begin{tikzpicture}[scale=0.48]
\begin{axis}[
    xmin = -6.5, xmax = 6.5,
    ymin = -1.5, ymax = 1.5]
    \addplot[
        domain = -6.5:-4,
        samples = 10,
        smooth,
        thick,
        blue,
    ] {1}; %{1+x+4};
    \addplot[
        domain = -4:-2,
        samples = 10,
        smooth,
        thick,
        blue,
    ] {1-(x+4)};
    \addplot[
        domain = -2:0,
        samples = 10,
        smooth,
        thick,
        blue,
    ] {-1+x+2};
    \addplot[
        domain = 0:2,
        samples = 10,
        smooth,
        thick,
        blue,
    ] {1-x};
    \addplot[
        domain = 2:4,
        samples = 10,
        smooth,
        thick,
        blue,
    ] {-1+x-2};
    \addplot[
        domain = 4:6.5,
        samples = 10,
        smooth,
        thick,
        blue,
    ] {1}; %{1-(x-4)};
\end{axis}
\end{tikzpicture}
\quad
\begin{tikzpicture}[scale=0.48]
\begin{axis}[
    xmin = -7.5, xmax = 7.5,
    ymin = -1.5, ymax = 1.5]
    \addplot[
        domain = -7.5:-5,
        samples = 10,
        smooth,
        thick,
        blue,
    ] {-1}; %{-1-(x+5)};
    \addplot[
        domain = -5:-3,
        samples = 10,
        smooth,
        thick,
        blue,
    ] {1+x+3};
    \addplot[
        domain = -3:-1,
        samples = 10,
        smooth,
        thick,
        blue,
    ] {1-(x+3)};
    \addplot[
        domain = -1:1,
        samples = 10,
        smooth,
        thick,
        blue,
    ] {-1+x+1};
    \addplot[
        domain = 1:3,
        samples = 10,
        smooth,
        thick,
        blue,
    ] {1-(x-1)};
    \addplot[
        domain = 3:5,
        samples = 10,
        smooth,
        thick,
        blue,
    ] {-1+x-3};
    \addplot[
        domain = 5:7.5,
        samples = 10,
        smooth,
        thick,
        blue,
    ] {1}; %{1-(x-5)};
\end{axis}
\end{tikzpicture}
\end{center}
\caption{Left: Plot of the function $f_1$ defined in \eqref{eq:f1_exp}, for the value $x_0 = 0.1$. Center: Plot of the function $f_2$ for $d = 4$ (defined in \eqref{eq:f_2_def_even}). Right: Plot of the function $f_2$ for $d = 5$ (defined in \eqref{eq:f_2_def_odd}).}
\label{fig:f_1_2}
\end{figure}
We define the discriminator $F : \R^d \to \R$ as
\begin{align} \label{eq:F_definition}
    F(x) = f_2\bigg(\sum_{i=1}^d f_1(x_i) \bigg).
\end{align}

\vspace{-6pt}
\paragraph{Construction of random variables $Z^{+}, Z^{-}$ with distributions $\mu_d, \nu_d$.} If $\xi^+, \xi^{-}$ are random vectors distributed uniformly over $\mathcal{B}^d_{+}$ and $\mathcal{B}^d_{-}$ respectively, and $X$ is a $d$-variate Gaussian $\mathcal{N}(0, \sigma^2 \text{Id})$, the variables $Z^{+} = \xi^+ + X$ and $Z^{-} = \xi^- + X$ are distributed according to $\mu_d$ and $\nu_d$ respectively. To see this, note that in analogy with $\rho_d = \pi_d * g_d$, we can write $\rho_d^{\pm} = \pi_d^{\pm} * g_d$, where $\pi_d^{\pm} = \frac{2}{4^d} \sum_{\beta \in \mathcal{B}^d_{\pm}} \prod_{i=1}^{d} \chi_{\beta_i} \delta_{\beta}$. Since $\xi^{\pm}$ are distributed according to $\pi^{\pm}_d$, and the law of a sum of random variables is the convolution of their distributions, the result follows. Thus, we can reexpress $\int_{\R^d} F(x) \, d(\mu_d-\nu_d)(x)$ as $\mathbb{E}[F(Z^{+})] - \mathbb{E}[F(Z^{-})]$.

\vspace{-6pt}
\paragraph{Lower-bounding $\mathbb{E}[F(Z^{+})] - \mathbb{E}[F(Z^{-})]$.} At this point, we take an arbitrary $0 < \epsilon < 1$, and define the sequence $(\sigma_d)_{d \geq 0}$ as the solutions of $\frac{x_0^2}{2\sigma_d^2} = \log(\frac{d \sigma_d}{\sqrt{2\pi}\epsilon x_0})$.
The solution $\sigma_d$ exists and is unique because the function $\sigma \mapsto \frac{x_0^2}{2\sigma^2}$ is strictly decreasing and bijective from $(0,+\infty)$ to $(0,+\infty)$, while the function $\sigma \mapsto \log(\frac{d \sigma}{\sqrt{2\pi}\epsilon x_0})$ is strictly increasing and bijective from $(0,\infty)$ to $\R$. The following result regarding the sequence $(\sigma_d)_d$ is shown in \autoref{sec:proofs_3_v_2}.
\begin{lemma} \label{lem:sigma_d}
If $(X_i)_{i=1}^{d}$ are independent random variables with distribution $\mathcal{N}(0,\sigma_d^2)$, we have that $P(\forall i \in \{1,\dots,d\}, \, X_i \leq x_0) \geq 1 - \epsilon$.
The sequence $(\sigma_d)_d$ is strictly decreasing, and $\sigma_d = \Omega(1/\log(d))$.
\end{lemma}
This allows us to prove an instrumental proposition concerning the values of $F$ at $Z^+$ and $Z^-$.
\begin{proposition} \label{prop:Z_pm_high_prob}
With probability at least $1-2\epsilon$, we have that simultaneously,
\begin{align} 
\begin{split}\label{eq:f2_f1}
    F(Z^{+}) 
    = 1 \quad \text{and} \quad 
    F(Z^{-}) = -1, \quad \text{when } d \equiv 0,1 \, (\text{mod } 4) \\
    F(Z^{+}) 
    = -1 \quad \text{and} \quad 
    F(Z^{-}) = 1, \quad \text{when } d \equiv 2,3 \, (\text{mod } 4)
\end{split}
\end{align}
Consequently, $\left| \mathbb{E}\left[ F(Z^+) \right] - \mathbb{E}\left[ F(Z^-) \right] \right| \geq 2-8\epsilon$.
\end{proposition}
\noindent\textbf{Proof sketch.}
By \autoref{lem:sigma_d}, with probability at least $1-2\epsilon$, $|X_i| \leq x_0$ for all $i \in \{1,\dots,d\}$. Equivalently, $|Z^+_i - \xi^+_i| \leq x_0$ and $|Z^-_i - \xi^-_i| \leq x_0$ for all $i \in \{1,\dots,d\}$. This implies that $f_1(Z^+_i) = \text{sign}(Z^+_i) = \text{sign}(\xi^+_i)$ and $f_1(Z^-_i) = \text{sign}(Z^-_i) = \text{sign}(\xi^-_i)$ for all $i \in \{1,\dots,d\}$. The statements \eqref{eq:f2_f1} follow from the definitions of the functions $f_2$ and the lower bound is a consequence of \eqref{eq:f2_f1} and the boundedness of $|F|$ (see full proof in \autoref{sec:proofs_3_v_2}). 
\qed
\paragraph{Bounding the path-norm of the discriminator $F$.} The following lemma, proved in \autoref{sec:proofs_3_v_2}, characterizes the discriminator $F$ as a three-layer network and provides bounds on its path-norms.
\begin{lemma} \label{lem:path_norm_F}
The function $F$ defined in \eqref{eq:F_definition} can be expressed as a three-layer ReLU neural network $f_{\mathcal{W}}$ of the form \eqref{eq:three_layer_nn} with widths $m_1 = 16 d$ and $m_2 = d+2$, with path-norms
\begin{align}
\begin{split}
    &\text{PN}_b(\mathcal{W}) \leq  \bigg( \frac{64}{x_0} + 1 \bigg) d^2 + 1 \text{ for $d$ even, and } \text{PN}_b(\mathcal{W}) \leq \bigg( \frac{64}{x_0} + 1 \bigg)d^2 + \frac{64d}{x_0} + 2 \text{ for $d$ odd.} \\
    &\text{PN}_{nb}(\mathcal{W}) = \frac{32 d^2}{x_0} \text{ for $d$ even, and } \text{PN}_{nb}(\mathcal{W}) = \frac{32 d^2 + 32 d}{x_0} \text{ for $d$ odd.}
\end{split}
\end{align}
\end{lemma}
We are in position to state the formal version of \autoref{thm:informal_1}.
\begin{theorem} \label{thm:separation_3_v_2_final}
    Setting $\epsilon = 1/8$ and $x_0 = 1/8$, we obtain that 
    \begin{align} \label{eq:d_f_2l_final_bound}
    d_{\mathcal{F}_{2L}}(\mu_d,\nu_d) &= O( %\max\left\{ 
    \kappa^d d^2 %, \log(d) e^{-18 d^2} \right\}
    ), \\
    d_{\mathcal{F}_{3L}}(\mu_d,\nu_d) &\geq \frac{1}{513 d^2 + 512 d + 1}. \label{eq:d_f_3l_final_bound}
    \end{align} 
\end{theorem}
\begin{proof}
    To prove \eqref{eq:d_f_2l_final_bound}, we plugged the bound $\sigma_d = \Omega(1/\log(d))$ from \autoref{lem:sigma_d} into \autoref{thm:upper_bound_2l}. We also used that for $\epsilon = 1/8$ and $x_0 = 1/8$, $\sigma_d \leq 1/6$ because at $1/6$, the curve $\sigma \mapsto \frac{x_0^2}{2\sigma^2}$ is below $\sigma \mapsto \log (\frac{d \sigma}{\sqrt{2\pi}\epsilon x_0} )$. Hence, $\sigma e^{-\frac{d^2}{2\sigma^2}} = O(\log(d) e^{-18 d^2}),$ which is $O( \kappa^d d^2)$. 
    To prove \eqref{eq:d_f_3l_final_bound}, we use that by \autoref{prop:Z_pm_high_prob}, $F$ is a three-layer neural network such that $|\mathbb{E}_{x \sim \mu_d}[F(x)] - \mathbb{E}_{x \sim \nu_d}[F(x)]| \geq 1$, and with path-norm with bias bounded by $513 d^2 + 512 d + 1$. Dividing the outermost layer weights by this quantity, we obtain a three-layer network with unit path-norm and the result follows.
    %\begin{align}
        %\frac{x_0^2}{2\sigma_d^2} &= \frac{6^2/8^2}{2} = \frac{9}{32}
        %\\ \log\left(\frac{d \sigma_d}{\sqrt{2\pi}\epsilon x_0} \right) &= \log\left(\frac{8*8}{2*6*\sqrt{2\pi}} \right) = \log\left(\frac{16}{3\sqrt{2\pi}} \right)
    %\end{align}
\end{proof}
Note that if we consider the discriminator class of three-layer networks with bounded path-norm without bias, \autoref{lem:path_norm_F} gives a lower bound of order $\Omega(1/d^2)$ as well. 

\vspace{-5pt}
\section{Separation between two-layer and RKHS discriminators}
\paragraph{The pair $(\mu_d,\nu_d)$.} For any $d \geq 0$, we define a pair of measures $\mu_d, \nu_d \in \mathcal{P}(\R^d)$ with densities $\frac{d\mu_d}{dx} = \rho^{+}_d$, $\frac{d\nu_d}{dx} = \rho^{+}_d$ such that 
\begin{align}
    \rho_d(x) := \frac{2 \sigma^d}{(2\pi)^{d/2}} e^{-\frac{\sigma^2 \|x\|^2}{2}} \sin\left(\ell x_1\right), \quad \text{where } x_1 = \langle x, e_1 \rangle.
\end{align}
Since $\int_{\R^d} \rho_d(x) \, dx = 0$ because $\rho_d$ is odd with respect to $x_1$, and $\int_{\R^d} |\rho_d(x)| \, dx \leq \frac{2 \sigma^d}{(2\pi)^{d/2}} \int_{\R^d} e^{-\frac{\sigma^2 \|x\|^2}{2}} \, dx = 2$, we have freedom in specifying $\rho^{+}_d, \rho^{-}_d$. If $\xi$ is the density of an arbitrary probability measure on $\R^d$, setting $\rho^{+}_d(x) = (1-\frac{1}{2}\int_{\R^d} |\rho_d(x)| \, dx) \xi(x) + \max\{0,\rho_d(x)\}$ and $\rho^{-}_d(x) = (1-\frac{1}{2}\int_{\R^d} |\rho_d(x)| \, dx) \xi(x) + \max\{0,-\rho_d(x)\}$ works. \autoref{fig:rho_density_sin_2_3} shows plots of $\rho_d$ for $d=1,2$. 
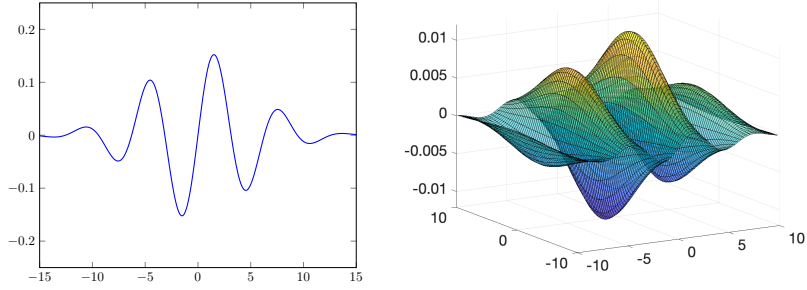
\begin{figure}[t]
\begin{center}
\begin{tikzpicture}[scale=0.5]
\begin{axis}[
    xmin = -15, xmax = 15,
    ymin = -0.25, ymax = 0.25]
    \addplot[
        domain = -15:15,
        samples = 200,
        smooth,
        thick,
        blue,
    ] {(2*0.2*exp(-0.2^2*x^2/2)*sin(deg(x)))/(sqrt(2*pi))};
    %{exp(-x/10)*( cos(deg(x)) + sin(deg(x))/10 )};
\end{axis}
\end{tikzpicture}
\quad
\includegraphics[scale=0.14]{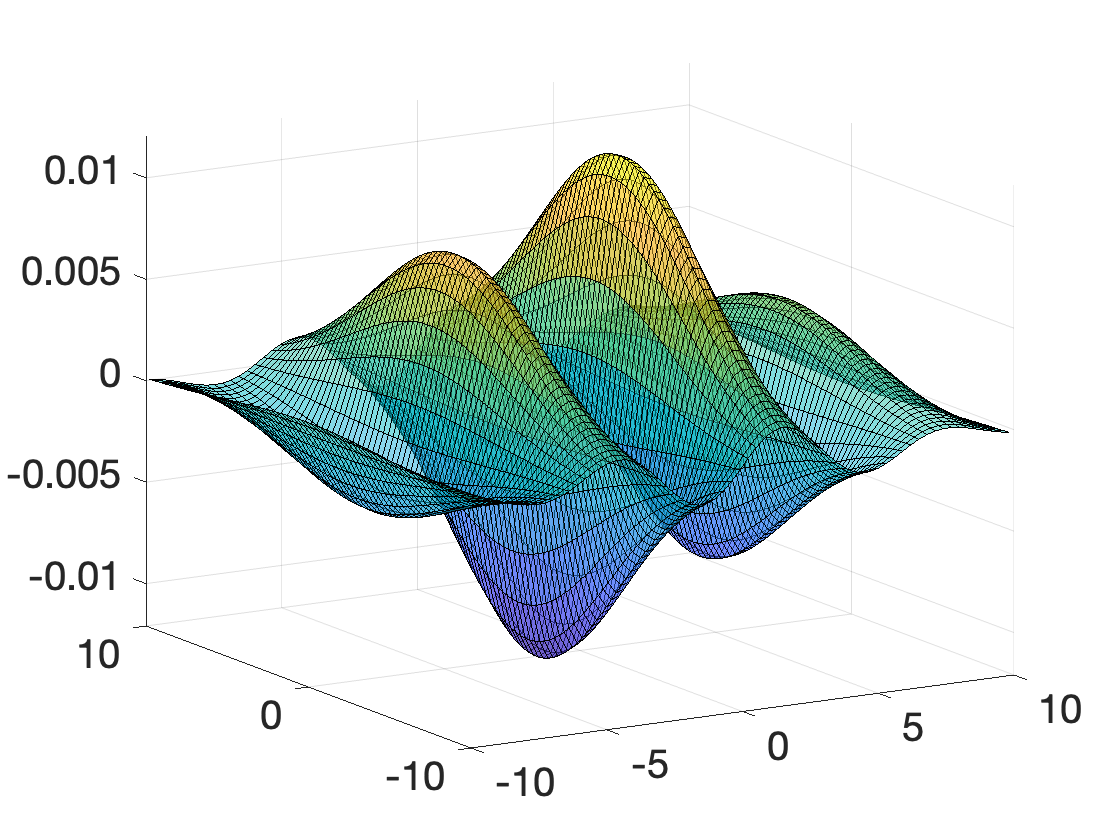}
\end{center}
\caption{Left: Plot of the density $\rho_{d}$ for $d=1$ with $\sigma = 0.2$ and $\ell = 1$. Right: Plot of the density $\rho_{d}$ for $d=2$ with $\sigma = 0.2$ and $\ell = 1$.}
\label{fig:rho_density_sin}
\end{figure}
The following lemma provides the Fourier transform of $\rho_d$. The proof in \autoref{sec:proofs_2_v_rkhs} involves using the convolution theorem; in this case $\rho_d$ is expressed as a product of functions and $\hat{\rho}_d$ is proportional to the convolution of their Fourier transforms.
\begin{lemma} 
The Fourier transform of $\rho_d$ reads $\widehat{\rho_d}(x) = \frac{i}{(2\pi)^{d/2}} (e^{-\frac{\|x+\ell e_1\|^2}{2\sigma^2}} - e^{-\frac{\|x-\ell e_1\|^2}{2\sigma^2}})$.
\end{lemma}
As in \autoref{sec:separation_3_v_2}, an application of \autoref{prop:inner_prod_fourier} shows that $\int_{\R^d} \rho_d(x) \sigma(\langle \theta, x \rangle - b) \, dx$ is equal to $(2\pi)^{(d-1)/2} \langle \hat{\sigma}(t), \widehat{\rho_d}(-t\theta) e^{-itb} \rangle$. Analogously, we use the expression of $\hat{\sigma}$ for the ReLU-like activations provided by \autoref{lem:relu_fourier}, and we obtain an explicit expression for $\int_{\R^d} \rho_d(x) \sigma(\langle \theta, x \rangle - b) \, dx$ from which the upper and lower bounds will follow:
\begin{align} \label{eq:discriminator_fourier}
    \frac{i}{\sqrt{2\pi}} \int_{\R} \left( A \frac{d^{\alpha}}{dt^{\alpha}} \left( \text{p.v.} \left[ \frac{1}{i\pi t} \right] \right) + B \frac{d^{\alpha}}{dt^{\alpha}} \delta(t) \right) \left( \left(e^{-\frac{\|t\theta-\ell e_1\|^2}{2\sigma^2}} - e^{-\frac{\|t\theta+\ell e_1\|^2}{2\sigma^2}} \right) e^{-itb} \right) \, dt, 
\end{align}
which can be simplified to (see \autoref{lem:2_v_rkhs_derivation} in \autoref{sec:proofs_2_v_rkhs}):
\begin{align} \label{eq:discriminator_fourier_2}
    \sqrt{\frac{2}{\pi}} i \left( -\frac{A \ell \theta_1}{\sigma^2} e^{-\frac{\ell^2}{2\sigma^2}} + B \int_{0}^{+\infty} \frac{ \sin(tb) (\exp(-\frac{\|t\theta -\ell e_1\|^2}{2\sigma^2}) - \exp(-\frac{\|t\theta +\ell e_1\|^2}{2\sigma^2}))}{t^2} \, dt \right)
\end{align}

\vspace{-6pt}
\paragraph{Upper bound for RKHS discriminators.} By equation \eqref{eq:H_ipm}, the IPM $d_{\mathcal{F}_{\mathcal{H}}}(\mu_d,\nu_d)$ corresponding to the unit ball of the RKHS $\mathcal{H}$ takes the form $(\int_{\mathbb{S}^{d-1} \times \R} (\int_{\R^d} \sigma \left(\langle \theta, x \rangle - b \right) \rho_d(x) \, dx )^2 \, d\tau(\theta,b) )^{1/2}$. Armed with the expression \eqref{eq:discriminator_fourier_2} for $\int_{\R^d} \sigma \left(\langle \theta, x \rangle - b \right) \rho_d(x) \, dx$, we proceed to upper-bound the absolute value of this expression in the following proposition proved in \autoref{sec:proofs_2_v_rkhs}.
\begin{proposition} \label{prop:upper_bound_rkhs}
We have that
\begin{align} \label{eq:upper_bound_rkhs_prop}
    d_{\mathcal{F}_{\mathcal{H}}}(\mu_d,\nu_d) = O \left( d^{1/4} \left( \frac{1}{2^{d/4}} +  e^{-\frac{\ell^2}{4\sigma^2}} \right) + \frac{\ell^2}{\sigma^4} e^{-\frac{(\ell-1)^2}{2\sigma^2}} + \left(\frac{\ell}{\sigma^2} + 1 \right) e^{-\frac{\ell^2}{2\sigma^2}} \right).
\end{align}
\end{proposition}
Evidently, the upper bound depends on the choices of the parameters $\ell$ and $\sigma$ as a function of $d$.

\vspace{-6pt}
\paragraph{Lower bound for two-layer discriminators.} Our approach to lower-bound the IPM $d_{\mathcal{F}_{2L}}(\mu_d,\nu_d)$ is to lower-bound $\int_{\R^d} \rho_d(x) \sigma(\langle \theta, x \rangle - b) \, dx$ for some well chosen $(\theta,b) \in \mathbb{S}^{d-1} \times \R$, via the expression \eqref{eq:discriminator_fourier_2}. The result is as follows:
\begin{proposition} \label{prop:lower_bound_2l}
    Define the ${(\ell_d)}_{d \geq 0}$ as $\ell_d = \sqrt{d}$ and ${(\sigma_d)}_{d \geq 0}$ as the sequence of solutions to $\frac{x_0^2}{2\sigma^2} = \log \left( \frac{\sqrt{2} d^2 \sigma}{\sqrt{\pi} x_0} \right)$, which fulfills $\sigma_d \geq K/\log(d)$. Then, $d_{\mathcal{F}_{2L}}(\mu_d,\nu_d) = \Omega \left( \frac{1}{d \log(d)}\right)$.
\end{proposition}
If we substitute the choices we made for $(\ell_d)$ and $(\sigma_d)$ into \eqref{eq:upper_bound_rkhs_prop}, we obtain
\begin{align}
    d_{\mathcal{F}_{\mathcal{H}}}(\mu_d,\nu_d) = O \bigg( d^{1/4} \left( \frac{1}{2^{d/4}} +  e^{-\frac{d}{16}} \right) + \frac{d e^{-\frac{(\sqrt{d}-1)^2}{16}}}{16} + \frac{\sqrt{d}+4}{4} e^{-\frac{d}{8}} \bigg) = O\bigg(d e^{-\frac{(\sqrt{d}-1)^2}{16}} \bigg),
\end{align}
which yields \autoref{thm:informal_2}.

\vspace{-5pt}
\section{Discussion} \label{sec:separation_2_v_rkhs}
\paragraph{Why small IPM values preclude discrimination of distributions from samples.} Suppose that $\mathcal{F}$ is a class of functions $\R^d \to \R$, and $\mu_n, \nu_n$ are empirical measures built from $n$ samples from $\mu, \nu$ respectively. Let $\hat{\mathcal{F}}_{\mu} = \{ f - \mathbb{E}_{x \sim \mu}[f(x)] \, | \, f \in \mathcal{F} \}$ be the recentered function class according to $\mu$ (analogous for $\nu$). Letting $\mathcal{R}_n(\mathcal{F}) = \mathbb{E}_{\sigma_i,x_i} \sup_{f \in \mathcal{F}} |\frac{1}{n} \sum_{i=1}^n \sigma_i f(x_i)|$ be the Rademacher complexity of $\mathcal{F}$, it turns out that
$\frac{1}{2} \mathcal{R}_n(\hat{\mathcal{F}}_{\mu}) \leq \mathbb{E}[\sup_{f \in \mathcal{F}} |\mathbb{E}_{x \sim \mu} [f(x)] - \mathbb{E}_{x \sim \mu_n} [f(x)]|] \leq 2 \mathcal{R}(\mathcal{F})$ (Proposition 4.11, \cite{wainwright2019high}),
and an application of McDiarmid's inequality shows that w.h.p. (with high probability), the IPM $d_{\mathcal{F}}(\mu, \mu_n) = \sup_{f \in \mathcal{F}} |\mathbb{E}_{x \sim \mu} [f(x)] - \mathbb{E}_{x \sim \mu_n} [f(x)]|$ does not lie far from these bounds. 

$\mathcal{F}$ is a useful discriminator class if $d_{\mathcal{F}}(\mu_n,\nu_n)$ is informative of the value of $d_{\mathcal{F}}(\mu,\nu)$ for a tractable data size $n$. This is \textit{not} the case if $d_{\mathcal{F}}(\mu,\nu)$ is negligible compared to $d_{\mathcal{F}}(\mu,\mu_n), d_{\mathcal{F}}(\nu,\nu_n)$ and their fluctuations, as the statistical noise dominates over the signal\footnote{Strictly speaking, if $d_{\mathcal{F}}(\mu,\nu)$ was smaller than $d_{\mathcal{F}}(\mu,\mu_n), d_{\mathcal{F}}(\nu,\nu_n)$ but greater or comparable to their fluctuations, $\mathcal{F}$ could potentially be an effective discriminator in some settings, but this situation seems implausible. To discard it formally, one may try to develop a kind of reverse McDiarmid inequality.}. Since $d_{\mathcal{F}}(\mu,\mu_n), d_{\mathcal{F}}(\nu,\nu_n)$ are w.h.p. of the order of $\frac{1}{2} \mathcal{R}_n(\hat{\mathcal{F}}_{\mu})$, and the classes $\mathcal{F}_{3L}, \mathcal{F}_{2L}, \mathcal{F}_{\mathcal{H}}$ studied in our paper (as well as their centered versions) have Rademacher complexities $\Theta(1/\sqrt{n})$ \citep{e2020banach}, we need to take $n$ of order $\Omega(1/d_{\mathcal{F}}(\mu,\nu)^2)$ to get decent discriminator performance. The required $n$ is prohibitively costly when $d_{\mathcal{F}}(\mu_d,\nu_d)$ is exponentially small in $d$, as in our cases.

\vspace{-6pt}
\paragraph{Can we make $\mu_d$ and $\nu_d$ any simpler in \autoref{sec:separation_3_v_2}?} One might wonder whether a simpler $\rho_d$ might suffice to show a separation result. Specifically, one might think of replacing $\mathcal{B} = \{\pm \frac{3}{2}, \pm \frac{1}{2}\}$ by $\mathcal{B} = \{ \pm 1 \}$. The upper bound on two-layer networks would not go through because the factor $\prod_{i=1}^{d} \cos \left( \omega_i \right) \sin \left( 2 \omega_i \right)$ would become $\prod_{i=1}^{d} \sin \left(\omega_i \right)$, which does not admit a uniform exponentially decreasing upper bound. Moreover, it can be seen that for $\sigma = O(1/\log(d))$, the two-layer network $f_2(\sum_{i=1}^{d} x_i)$ would be able to discriminate between $\mu_d$ and $\nu_d$. 

\vspace{-6pt}
\paragraph{Do our arguments work for other activation functions and weight norms?} Our proofs makes use of the specific form of the Fourier transform of the ReLU and leaky ReLU. One may try to apply the same method for other activation functions via their Fourier transforms; intuitively one should be able to obtain exponentially decreasing lower bounds as well, because the factor $\prod_{i=1}^{d} \cos \left( \omega_i \right) \sin \left( 2 \omega_i \right)$ will show up in some way or another. If we use different norms to define the three-layer and two-layer IPMs, the results are unchanged up to polynomial factors because weight norms are equivalent to each other up to polynomial factors in $d$ (using that the widths of our networks are polynomial in $d$). Finally, it would be interesting to adapt our upper bound for the MMD to slightly different kernels such as the neural tangent kernel (NTK, \cite{jacot2018neural}).

%Bounds for other neural network norms.

%How do exponentially small IPMs imply that distributions cannot be distinguished?

%Comment on the similarity of the adaptivity result with our previous paper.

\bibliography{biblio}

\newpage

\appendix

\section{IPM derivations} \label{sec:ipm_derivations}
Define the function class $G_{3L}$ of ReLU neural networks of the form $g(x) = \pm \sigma (\sum_{j=1}^{m_1} w_j \sigma(\langle \theta_j, x \rangle - b_j) + w_0 )$ such that $\sum_{j=1}^{m_1} |w_j| \cdot \|(\theta_j,b_j)\|_2 + |w_0| \leq 1$. 
\begin{lemma} \label{lem:G_3l}
Any function in $F_{3L}$ may be written as a convex combination of functions in $G_{3L}$ and the constant function 1. 
\end{lemma}
\begin{proof}
Let 
\begin{align}
f_{\mathcal{W}}(x) = \sum_{i=1}^{m_2} w_i \sigma \left(\sum_{j=1}^{m_1} W_{i,j} \sigma\left(\langle \theta_j, x \rangle - b_j \right) + W_{i,0} \right) + w_0. 
\end{align}
belong to $F_{3L}$, which means that $\text{PN}_b(\mathcal{W}) = \sum_{i=1}^{m_2} |w_i| ( \sum_{j=1}^{m_1} |W_{i,j}| \cdot \|(\theta_j,b_j)\|_2 + |W_{i,0}|) + |w_0| \leq 1$.
We may renormalize the weights such that $\sum_{j=1}^{m_1} |W_{i,j}| \cdot \|(\theta_j,b_j)\|_2 + |W_{i,0}| = 1$ for all $i$, by moving the appropriate factors outside of the ReLU activation thanks to the 1-homogeneity. Then, $\text{PN}_b(\mathcal{W}) = \sum_{i=0}^{m_2} |w_i| \leq 1$. We may further renormalize the weights such that $\sum_{i=0}^{m_2} |w_i| = 1$ and $\sum_{j=1}^{m_1} |W_{i,j}| \cdot \|(\theta_j,b_j)\|_2 + |W_{i,0}| \leq 1$ for all $i$.

Setting $g_i(x) = \text{sign}(w_i) \sigma \left(\sum_{j=1}^{m_1} W_{i,j} \sigma\left(\langle \theta_j, x \rangle - b_j \right) + W_{i,0} \right)$, we obtain the expression $f_{\mathcal{W}}(x) = \sum_{i=1}^{m_2} |w_i| g_i(x) + |w_0|$. That is, $f_{\mathcal{W}}$ can be written as a convex combination of $\{g_i\}_{i=1}^{m_2}$ and the constant function 1. Note that $g_i$ belongs to $G_{3L}$, which concludes the proof. 
\end{proof}

Since $f \mapsto \pm(\mathbb{E}_{x \sim \mu}[f(x)] - \mathbb{E}_{x \sim \nu}[f(x)])$ are concave mappings, their suprema over $G_{3L}$ is equal to their suprema over the convex hull $\text{conv}(G_{3L})$. Since $\mathbb{E}_{x \sim \mu}[f(x)] - \mathbb{E}_{x \sim \nu}[f(x)]$ is 0 when $f$ is a constant function, by \autoref{lem:G_3l} the suprema over $\text{conv}(G_{3L})$ are equal to the suprema over $F_{3L}$, which concludes the proof of equation \eqref{eq:3L_ipm}. Equation \eqref{eq:2L_ipm} follows from a similar argument. Equation \eqref{eq:H_ipm} is derived using the proof of Lemma 2 of \cite{domingoenrich2021separation}. 

\vspace{-10pt}
\section{Proofs of \autoref{sec:separation_3_v_2}} \label{sec:proofs_3_v_2}
\noindent\textbf{Proof of \autoref{lem:rho_d_conv}.}
If we take a Schwartz function $\phi \in \mathcal{S}(\R^d)$, we have
\begin{align}
\begin{split}
    &\langle \pi_d*g_d, \phi \rangle = \langle g_d(y), \langle \phi(x + y), \pi_d(x) \rangle \rangle \\ &= \int_{\R^d} \frac{1}{(2\pi \sigma^2)^{d/2}} e^{-\frac{\|y\|^2}{2\sigma^2}} \left\langle \phi(x+y), \frac{2}{4^d} \sum_{\beta \in \mathcal{B}^d} \prod_{i=1}^{d} \chi_{\beta_i} \delta_{\beta} \right\rangle \, dy \\ &= \frac{2}{(4\sqrt{2\pi \sigma^2})^d} \int_{\R^d} e^{-\frac{\|y\|^2}{2\sigma^2}} \sum_{\beta \in \mathcal{B}^d} \prod_{i=1}^{d} \chi_{\beta_i} \phi(y+\beta) \, dy \\ &= \frac{2}{(4\sqrt{2\pi \sigma^2})^d} \sum_{\beta \in \mathcal{B}^d} \prod_{i=1}^{d} \chi_{b_i} \int_{\R^d} e^{-\frac{\|y\|^2}{2\sigma^2}} \phi(y+\beta) \, dy 
    \\ &= \frac{2}{(4\sqrt{2\pi \sigma^2})^d} \sum_{\beta \in \mathcal{B}^d} \prod_{i=1}^{d} \chi_{b_i} \int_{\R^d} e^{-\frac{\|\tilde{y}-\beta\|^2}{2\sigma^2}} \phi(\tilde{y}) \, d\tilde{y} \\ &=
    \frac{2}{(4\sqrt{2\pi \sigma^2})^d} \int_{\R^d} \sum_{\beta \in \mathcal{B}^d} \prod_{i=1}^{d} \chi_{b_i} e^{-\frac{\|\tilde{y}-\beta\|^2}{2\sigma^2}} \phi(\tilde{y}) \, d\tilde{y} = \langle \rho_d, \phi \rangle.
\end{split}
\end{align}
\qed

\noindent\textbf{Proof of \autoref{prop:inner_prod_fourier}.}
We adapt the argument of Lemma 3 of \cite{domingoenrich2021separation}. Define $\sigma_b : \R \to \R$ as the translation of $\sigma$ by $-b$, i.e. $\sigma_b(x) = \sigma(x-b)$. Note that $\widehat{\sigma_b}(\omega) = e^{-ib\omega} \hat{\sigma}(\omega)$.
We have
\begin{align}
\begin{split}
    &\int_{\R^d} \phi(x) \sigma(\langle \theta, x \rangle - b) \, dx = \int_{\R^d} \phi(x) \sigma_b(\langle \theta, x \rangle) \, dx = \int_{\R^d} \left(\frac{1}{(2\pi)^{d/2}} \int_{\R^d} \hat{\phi}(\omega) e^{i \langle \omega, x \rangle} \, d\omega \right) \sigma_b(\langle \theta, x \rangle) \, dx \\ &= \int_{\text{span}(\theta)} \int_{\text{span}(\theta)^{\perp}} \left(\frac{1}{(2\pi)^{d/2}} \int_{\text{span}(\theta)^{\perp}} \left( \int_{\text{span}(\theta)} \hat{\phi}(\omega) e^{i \langle \omega_\theta, x_{\theta} \rangle} \, d\omega_{\theta} \right) e^{i \langle \omega_{\theta^{\perp}}, x_{\theta^{\perp}} \rangle} \, d\omega_{\theta^{\perp}} \right) \, dx_{\theta^{\perp}} \sigma_b(\langle \theta, x_{\theta} \rangle) \, dx_{\theta} \\ &= \frac{1}{\sqrt{2\pi}} \int_{\text{span}(\theta)} (2\pi)^{(d-1)/2} \int_{\text{span}(\theta)} \hat{\phi}(\omega_{\theta}) e^{i \langle \omega_\theta, x_{\theta} \rangle} \, d\omega_{\theta} \, \sigma_b(\langle \theta, x_{\theta} \rangle) \, dx_{\theta} \\ &= (2\pi)^{(d-1)/2} \int_{\R} \frac{1}{\sqrt{2\pi}} \int_{\R} \hat{\phi}(t\theta) e^{it x} \, dt \, \sigma_b(x) \, dx = (2\pi)^{(d-1)/2} \int_{\R} \frac{1}{\sqrt{2\pi}} \int_{\R} \hat{\phi}(t\theta) e^{it x} \, dt \, \sigma(x-b) \, dx \\ &= (2\pi)^{(d-1)/2} \int_{\R} \frac{1}{\sqrt{2\pi}} \int_{\R} \hat{\phi}(t\theta) e^{it (\tilde{x} + b)} \, dt \, \sigma(\tilde{x}) \, dx = (2\pi)^{(d-1)/2} \langle \check{\sigma}(t), \hat{\phi}(t\theta) e^{itb} \rangle. 
\end{split}
\end{align}
In the third equality, we
rewrite $\R^{d} = \text{span}(\theta) + \text{span}(\theta)^{\perp}$ and we use Fubini’s theorem twice. In the fourth equality we
use the following argument: denoting $h(x_{\theta^{\perp}}, \omega_{\theta}) = \int_{\text{span}} \hat{\phi}(\omega_{\theta} + \omega_{\theta}^{\perp}) e^{i \langle \omega_\theta, x_{\theta} \rangle} \, d\omega_{\theta}$, we have that 
\begin{align}
\begin{split}
&\int_{\text{span}(\theta)^{\perp}} \left( \int_{\text{span}(\theta)^{\perp}} h(x_{\theta^{\perp}}, \omega_{\theta}) e^{i \langle \omega_{\theta^{\perp}}, x_{\theta^{\perp}} \rangle} \, d\omega_{\theta^{\perp}} \right) dx_{\theta^{\perp}} = (2\pi)^{(d-1)/2} \int_{\text{span}(\theta)^{\perp}} \hat{h}(-\omega_{\theta^{\perp}}, \omega_{\theta}) \, d\omega_{\theta^{\perp}} \\ &= (2\pi)^{d-1} h(0,\omega_x) = (2\pi)^{d-1} \int_{\text{span}(\theta)} \hat{\phi}(\omega_{\theta}) e^{i \langle \omega_\theta, x_{\theta} \rangle} \, d\omega_{\theta}.
\end{split}
\end{align}
To conclude the proof, note that for any test function $\phi \in \mathcal{S}(\R)$, $\langle \check{\sigma}(x), \phi(x) \rangle = \langle \sigma(x), \check{\phi}(x) \rangle = \int_{\R} \sigma(x) \frac{1}{\sqrt{2\pi}} \, \int_{\R} e^{itx} \phi(t) \, dt \, dx = \int_{\R} \sigma(x) \frac{1}{\sqrt{2\pi}} \, \int_{\R} e^{-i(-t)x} \phi(t) \, dt \, dx = \int_{\R} \sigma(x) \frac{1}{\sqrt{2\pi}} \, \int_{\R} e^{-itx} \phi(-t) \, dt \, dx = \langle \hat{\sigma}(x), \phi(-x) \rangle$.
\qed

\vspace{2pt}
\noindent\textbf{Proof of \autoref{lem:pv_upper_bound}.}
Recall that $\text{p.v.} \left[ \frac{1}{\omega} \right] (u) = \int_{0}^{+\infty} \frac{u(\omega) - u(-\omega)}{\omega}$. On the one hand,
\begin{align}
    \left| \int_{0}^{1} \frac{u(x) - u(-x)}{x} \, dx \right| \leq \int_{0}^{1} \frac{|u(x) - u(-x)|}{x} \, dx \leq \int_{0}^{1} \frac{2x}{x} \sup_{y \in (-1,1)} |u'(y)| \, dx = 2 \sup_{y \in (-1,1)} |u'(y)|.
\end{align}
On the other hand,
\begin{align}
\begin{split}
    \left| \int_{1}^{+\infty} \frac{u(x) - u(-x)}{x} \, dx \right| &\leq \int_{1}^{+\infty} \frac{(|u(x)| + |u(-x)|)x^{\delta}}{x^{1+\delta}} \, dx \\ &\leq \int_{1}^{+\infty} 2 \left( \sup_{y \in \R \setminus [-1,1]} |u(y) \cdot y^{\delta}| \right) \frac{1}{x^{1+\delta}} \, dx = \frac{2}{\delta} \left( \sup_{y \in \R \setminus [-1,1]} |u(y) \cdot y^{\delta}| \right).
\end{split}
\end{align}
\qed
\begin{lemma} \label{lem:sup_bound_u_u_prime}
Let $u : \R \to \mathbb{C}$ defined by \eqref{eq:u_def}. Then,
\begin{align}
\begin{split}
    &\sup_{x \in \R} |u'(x)| \leq O\left( \kappa^d \left( d^2 + d |b| + b^2 \right) \right)\\
    &\sup_{x \in \R} |u(x) \cdot x| \leq O\left( \kappa^d \left(\frac{d + |b|}{\sigma} \right) \right)
\end{split}
\end{align}
\end{lemma}
\begin{proof}
Note that
\begin{align}
\begin{split} \label{eq:u_prime}
    u'(t) &= \left(-\sigma^2 t - ib - \sum_{i=1}^d \frac{\theta_i \sin(t \theta_i)}{\cos(t \theta_i)} + 2 \sum_{i=1}^d \frac{\theta_i \cos(2t \theta_i)}{\sin(2t \theta_i)} \right)^2 e^{-\frac{\sigma^2 t^2}{2} - itb} \prod_{i=1}^d \cos(t \theta_i) \sin(2 t \theta_i) \\ &+ \left(-\sigma^2 - \sum_{i=1}^d \frac{\theta_i^2}{\cos^2(t \theta_i)} - 2^2 \sum_{i=1}^d \frac{\theta_i^2}{\sin^2(2t \theta_i)} \right) e^{-\frac{\sigma^2 t^2}{2} - itb} \prod_{i=1}^d \cos(t \theta_i) \sin(2 t \theta_i)
\end{split}
\end{align}
Remark that 
\begin{align}
    \left(- \frac{\theta_i \sin(t \theta_i)}{\cos(t \theta_i)} \right)^2 - \frac{\theta_i^2}{\cos^2(t \theta_i)} = - \theta_i^2, \quad \left( 2 \frac{\theta_i \cos(2t \theta_i)}{\sin(2t \theta_i)} \right)^2 - 2^2 \frac{\theta_i^2}{\sin^2(2t \theta_i)} = - 2^2 \theta_i^2
\end{align}
and that $\sum_{i} \theta_i^2 = \|\theta\|^2 = 1$.
Hence, equation \eqref{eq:u_prime} may be rewritten as $e^{-\frac{\sigma^2 t^2}{2} - itb} \prod_{i=1}^d \cos(t \theta_i) \sin(2 t \theta_i)$ times
\begin{align}
    & %(\sigma^2 t + ib)^2 - \sigma^2 
    \sigma^4 t^2 + 2 i b \sigma^2 t - b^2 - 5 + 2(\sigma^2 t + ib)\left( \sum_{i=1}^d \frac{\theta_i \sin(t \theta_i)}{\cos(t \theta_i)} - 2 \sum_{i=1}^d \frac{\theta_i \cos(2t \theta_i)}{\sin(2t \theta_i)} \right)  
    \\ & -4 \sum_{i,j=1}^d \frac{\theta_i \theta_j \sin(t \theta_i) \cos(2t \theta_j)}{\cos(t \theta_i) \sin(2t \theta_j)} + \sum_{\substack{i,j=1 \\ i \neq j}}^{d} \frac{\theta_i \theta_j \sin(t \theta_i) \sin(t \theta_j)}{\cos(t \theta_i) \cos(t \theta_j)} + 4 \sum_{\substack{i,j=1 \\ i \neq j}}^{d} \frac{\theta_i \theta_j \cos(2t \theta_i) \cos(2t \theta_j)}{\sin(2t \theta_i) \sin(2t \theta_j)} 
    %(\sigma^2 t + ib)^2 = \sigma^4 t^2 + 2 i b \sigma^2 t - b^2
\end{align}
The functions $t \mapsto |\cos(t \theta_i) \sin(2 t \theta_i)|$ are upper-bounded by 0.77 on $\R$ regardless of the value of $\theta_i$. To see this, define $x = t \theta_i$. Hence, $|\cos(t \theta_i) \sin(2 t \theta_i)| = |\cos(x) \sin(2x)|$. \autoref{lem:bound_cos_sin} shows that $\kappa := \sup_{x \in \R} |\cos(x) \sin(2x)| = 0.7698\dots$. The following upper bounds hold for all $t \in \R$:
\begin{align}
\begin{split}
    \left|
    \prod_{i=1}^d \cos(t \theta_i) \sin(2 t \theta_i) \right| &\leq \kappa^d, \quad \left|t e^{-\frac{\sigma^2 t^2}{2} - itb} 
    \right| \leq \max_{x \in \R} \{ x e^{-\frac{\sigma^2 x^2}{2}} \} = \frac{1}{\sigma \sqrt{e}}, \\
    \left|t^2 e^{-\frac{\sigma^2 t^2}{2} - itb} 
    \right| &\leq 
    \max_{x \geq 0} \{ x e^{-\frac{\sigma^2 x}{2}} \} = \frac{2}{e \sigma^2}.
\end{split}    
\end{align}
Thus, the following is a crude upper bound of $|u'(t)|$ for any $t \in \R$:
\begin{align}
    &\kappa^d \left( \left( \frac{2\sigma^2}{e} + 5 + b^2 + \frac{6d\sigma}{\kappa \sqrt{e}} + \frac{4 d^2}{\kappa^2} + \frac{d(d-1)}{\kappa^2} + \frac{4 d(d-1)}{\kappa^2} \right)^{2} + \left( \frac{2b\sigma}{\sqrt{e}} + \frac{6d |b|}{\kappa} \right)^{2} \right)^{1/2} \\ &= O\left( \kappa^d \left( d^2 + d |b| + b^2 \right) \right)
\end{align}
In the last $O$-notation expression we have only kept the relevant variables: $\sigma$ is relevant because it appears in the numerator and we will take it smaller than 1. 
Similarly, the following is an upper bound on $|t \cdot u(t)|$ for any $t \in \R$:
\begin{align}
    \kappa^d \left( \left( \frac{2}{e} + \frac{2 d a}{\sigma \kappa \sqrt{e}} + \frac{2 d}{\sigma \kappa \sqrt{e}} \right)^2 + \left( \frac{b}{\sigma \sqrt{e}} \right)^2 \right)^{1/2} = O\left( \kappa^d \left(\frac{d + |b|}{\sigma} \right) \right)
\end{align}
\end{proof}
\begin{lemma} \label{lem:bound_cos_sin}
    The function $h(x) = \cos(x) \sin(2x)$ satisfies
    \begin{align}
        \max_{x \in \R} |h(x)| = 0.769800358917917...
    \end{align}
\end{lemma}
\begin{proof}
First note that $h$ has period $2\pi$, which means that we can restrict the search of maximizers to $[-\pi,\pi]$.  
We have that $h'(x) = -\sin(x) \sin(2x) + 2 \cos(x) \cos(2x)$. The condition $h'(x^{*}) = 0$ is necessary for $x^{*}$ to be a local maximizer of $|h|$, and it may be rewritten as $\tan(x) = 2 \, \text{cotan}(2x)$. Remark that $x \to \tan(x)$ is increasing and bijective from $(\pi(z-1/2),\pi(z+1/2))$ to $\R$, and that $x \to 2 \, \text{cotan}(2x)$ is decreasing and bijective from $(\frac{\pi z}{2},\frac{\pi (z+1)}{2})$ to $\R$ for any $z \in \mathbb{Z}$. Thus, there exist $6$ solutions of $h'(x)$ on $[-\pi,\pi]$: one for each interval $(\frac{\pi z}{2},\frac{\pi (z+1)}{2})$ for $z = -2, \dots, 1$, and additional solutions at $\frac{-\pi}{2}$ and at $\frac{\pi}{2})$, where both $\tan(x)$ and $2 \, \text{cotan}(2x)$ take value $+\infty$ and $-\infty$ respectively. With this information, any algorithm that finds local maximizers over intervals allows us to compute the global maximum of $|h|$, which is equal to $0.769800358917917...$, and is attained, among other points, at $0.615478880595691...$
\end{proof}
\begin{figure}[h]
\begin{center}
\begin{tikzpicture}[scale=0.7]
\begin{axis}[
    xmin = -4.5, xmax = 4.5,
    ymin = -1.3, ymax = 1.3]
    \addplot[
        domain = -4.5:4.5,
        samples = 300,
        smooth,
        thick,
        blue,
    ] {cos(deg(x))*sin(2*deg(x))};
    \draw [thick, draw=red] 
        (axis cs: -4.5,0.77) -- (axis cs: 4.5,0.77)
        node[pos=0.5, above] {$y=0.77$};
    \draw [thick, draw=red] 
        (axis cs: -4.5,-0.77) -- (axis cs: 4.5,-0.77)
        node[pos=0.5, below] {$y=-0.77$};
    %{exp(-x/10)*( cos(deg(x)) + sin(deg(x))/10 )};
\end{axis}
\end{tikzpicture}
\end{center}
\caption{Plot of the function $x \mapsto \cos(x) \sin(2x)$.}
\label{fig:sin_cos}
\end{figure}
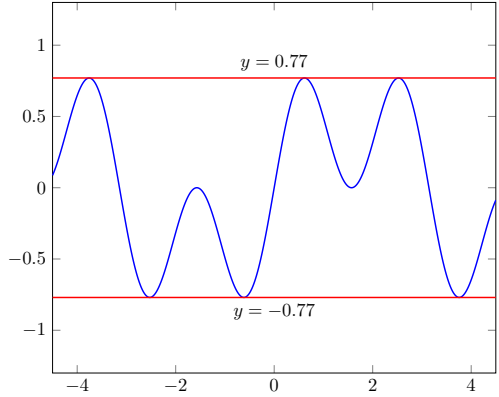
\noindent\textbf{Proof of \autoref{lem:b_d_sqrtd}.}
Note that for all $\beta \in \{\pm 1 \}^d$, $\|\beta\|=\sqrt{d}$. If $b > d+\sqrt{d}$, for any $\beta$ we have that $b - \langle \theta, \beta \rangle \geq b - \|\theta\| \|\beta\| \geq d + \sqrt{d} - \sqrt{d} = d$.
%the distance between the half-plane $\{ x \in \R^d |  \langle \theta, x \rangle - b \geq 0 \}$ and the origin is greater than $d +\sqrt{d}$, which means that the distance from the half-plane an the set $\{\pm 1 \}^d$ is at least $d$. 
Thus, using the notation $\chi_{\beta} = \prod_{i=1}^{d} \chi_{\beta_i}$ we have that
\begin{align}
\begin{split} \label{eq:big_b}
	&\int_{\R^d} \rho_d(x) \sigma(\langle \theta, x \rangle - b) \, dx \\ &= \frac{1}{\left( 2 \sqrt{2 \pi \sigma^2} \right)^{d}}\sum_{\beta \in \mathcal{B}^d} \chi_{\beta} \int_{\R^d} e^{-\frac{\|x-\beta\|^2}{2\sigma^2}} \sigma(\langle \theta, x \rangle - b) \, dx \\ &= \frac{1}{\left( 2 \sqrt{2 \pi \sigma^2} \right)^{d}}\sum_{\beta \in \mathcal{B}^d} \chi_{\beta} \int_{\R^d} e^{-\frac{\|\tilde{x}\|^2}{2\sigma^2}} \sigma(\langle \theta, \tilde{x} + \beta \rangle - b) \, dx \\ &= \frac{1}{\left( 2 \sqrt{2 \pi \sigma^2} \right)^{d}}\sum_{\beta \in \mathcal{B}^d} \chi_{\beta} \int_{b-\langle \theta, \beta \rangle}^{+\infty} (t - (b-\langle \theta, \beta \rangle)) e^{-\frac{t^2}{2\sigma^2}} \, dt \int_{\R^{d-1}} e^{-\frac{\|x\|^2}{2\sigma^2}} \, dx \\ &= \frac{1}{2^{d} \sqrt{2 \pi \sigma^2}} \sum_{\beta \in \mathcal{B}^d} \chi_{\beta} \int_{b-\langle \theta, \beta \rangle}^{+\infty} (t - (b-\langle \theta, \beta \rangle)) e^{-\frac{t^2}{2\sigma^2}} \, dt \\ &\leq \frac{1}{2^{d} \sqrt{2 \pi \sigma^2}} \sum_{\beta \in \mathcal{B}^d} \chi_{\beta} \int_{d}^{+\infty} (t - d) e^{-\frac{t^2}{2\sigma^2}} \, dt
	\\ &\leq \frac{\sigma}{\sqrt{2\pi}} e^{-\frac{d^2}{2\sigma^2}}
\end{split}
\end{align}
In the second equality we used the change of variables $\tilde{x} = x - \beta$. The first inequality holds because $b - \langle \theta, \beta \rangle \geq d$, and the second inequality holds because $\frac{1}{\sqrt{2 \pi \sigma^2}} \int_{d}^{+\infty} (t - d) e^{-\frac{t^2}{2\sigma^2}} \, dt \leq \frac{1}{\sqrt{2 \pi \sigma^2}} \int_{d}^{+\infty} t e^{-\frac{t^2}{2\sigma^2}} \, dt \leq \frac{\sigma}{\sqrt{2\pi}} e^{-\frac{d^2}{2\sigma^2}}$ by \autoref{lem:gaussian_tails}.
In the case $b < -d - \sqrt{d}$, the same argument implies that $\int_{\R^d} \rho_d(x) \sigma(\langle \theta, x \rangle - b) \, dx$ is equal to 
\begin{align}
\begin{split} \label{eq:b_less_than_minus_d}
	 \frac{1}{ 2^d \sqrt{2 \pi \sigma^2}} \sum_{\beta \in \mathcal{B}^d} \chi_{\beta} \left( \int_{\R} (t - (b-\langle \theta, \beta \rangle)) e^{-\frac{t^2}{2\sigma^2}} \, dt - \int_{-\infty}^{b-\langle \theta, \beta \rangle} (t - (b-\langle \theta, \beta \rangle)) e^{-\frac{t^2}{2\sigma^2}} \, dt \right)
\end{split}
\end{align}
An application of \autoref{lem:rho_d_moments} yields
\begin{align}
\begin{split}
	&\frac{1}{ 2^d \sqrt{2 \pi \sigma^2}} \sum_{\beta \in \mathcal{B}^d} \chi_{\beta} \int_{\R} (t - (b-\langle \theta, \beta \rangle)) e^{-\frac{t^2}{2\sigma^2}} \, dt = \int_{\R^d} \rho_d(x) (\langle \theta, x \rangle - b) \, dx \\ &= \left\langle \theta, \int_{\R^d} x \rho_d(x) \, dx \right\rangle - b \int_{\R^d} \rho_d(x) \, dx = 0,
\end{split}
\end{align}
which means that \eqref{eq:b_less_than_minus_d} simplifies to
\begin{align}
	%\frac{1}{\left( 2 \sqrt{2 \pi \sigma^2} \right)^{d}}\sum_{b \in \mathcal{B}^d} \chi_b \int_{-\infty}^{b} (b-t) e^{-\frac{t^2}{2\sigma^2}} \, dt \, \int_{\R^{d-1}} e^{-\frac{\|x\|^2}{2\sigma^2}} \, dx \leq \frac{\sigma}{\sqrt{2\pi}} e^{-\frac{d^2}{2\sigma^2}}.
	\frac{1}{ 2^d \sqrt{2 \pi \sigma^2}} \sum_{\beta \in \mathcal{B}^d} \chi_{\beta} \int_{-\infty}^{b-\langle \theta, \beta \rangle} (b-\langle \theta, \beta \rangle - t) e^{-\frac{t^2}{2\sigma^2}} \, dt \leq \frac{\sigma}{\sqrt{2\pi}} e^{-\frac{d^2}{2\sigma^2}}.
\end{align}
Here, the inequality follows from the same argument as equation \eqref{eq:big_b}.
\qed
\begin{lemma}[Simple tail bounds for Gaussian distribution] \label{lem:gaussian_tails}
If $X \sim \mathcal{N}(0,\sigma^2)$, for all $x > 0$ we have $P(X \geq x) \leq \frac{\sigma}{x\sqrt{2\pi}} e^{-\frac{x^2}{2\sigma^2}}$, and $\mathbb{E}[X \mathds{1}_{X \geq x}] \leq \frac{\sigma}{\sqrt{2\pi}} e^{-\frac{x^2}{2\sigma^2}}$.
\end{lemma}
\begin{proof}
We write
\begin{align}
\begin{split}
    P(X \geq x) &= \frac{1}{\sqrt{2\pi\sigma^2}} \int_{x}^{+\infty} e^{-\frac{t^2}{2\sigma^2}} \, dt \leq \frac{1}{\sqrt{2\pi\sigma^2}} \int_{x}^{+\infty} \frac{t}{x} e^{-\frac{t^2}{2\sigma^2}} \, dt = \frac{2 \sigma^2}{x\sqrt{2\pi \sigma^2}} \int_{\frac{x}{\sqrt{2\sigma^2}}}^{+\infty} y e^{-y^2} \, dy \\ &= \frac{\sigma}{x\sqrt{2\pi}} \int_{\frac{x}{\sqrt{2\sigma^2}}}^{+\infty} 2y e^{-y^2} \, dy = \frac{\sigma}{x\sqrt{2\pi}} \int_{\frac{x^2}{2\sigma^2}}^{+\infty} e^{-\tilde{y}} \, dy = \frac{\sigma}{x\sqrt{2\pi}} e^{-\frac{x^2}{2\sigma^2}}
\end{split}
\end{align}
where we used the changes of variables $y = \frac{t}{\sqrt{2\sigma^2}}$ (i.e. $t = \sqrt{2 \sigma^2} y$), and $\tilde{y} = y^2$. Similarly, $\mathbb{E}[X \mathds{1}_{X \geq x}] = \frac{\sigma}{\sqrt{2\pi}} e^{-\frac{x^2}{2\sigma^2}}$.
\end{proof}
\begin{lemma} \label{lem:rho_d_moments}
	We have that $\int_{\R^d} x \rho_d(x) \, dx = 0$ and $\int_{\R^d} \rho_d(x) \, dx = 0$.
\end{lemma}
\begin{proof}
We use the short-hand $\tilde{\rho}(x) = \frac{1}{4\sqrt{2\pi \sigma^2}} \sum_{\beta \in \mathcal{B}} \chi_{\beta} \exp (-\frac{(x_i-\beta)^2}{2\sigma^2})$. Note that $\rho_d(x) = 2 \prod_{i=1}^{d} \tilde{\rho}(x_i)$.
By the definition of $\rho_d$,
\begin{align}
\begin{split}
	\int_{\R^d} x \rho_d(x) \, dx &= \int_{\R^d} \left( \sum_{j=1}^{d} x_j e_j \right) \rho_d(x) \, dx =  2 \int_{\R^d} \left( \sum_{j=1}^{d} x_j e_j \right) \prod_{i=1}^d \tilde{\rho}(x_i) \, dx \\ &= 2 \sum_{j=1}^{d} \int_{\R} x_j e_j \tilde{\rho}(x_j) \, dx_j \prod_{i \neq j} \int_{\R} \tilde{\rho}(x_i) \, dx_i = 0,
\end{split}
\end{align}
which holds because $\int_{\R} \tilde{\rho}(x_i) \, dx_i = 0$ as $\tilde{\rho}$ is an odd function. Similarly, we have that $\int_{\R^d} \rho_d(x) \, dx = 2 \prod_{i=1}^d \int_{\R} \tilde{\rho}(x_i) \, dx_i = 0$.
\end{proof}
\noindent\textbf{Proof of \autoref{lem:sigma_d}.}
If $(X_i)_{i=1}^{d}$ are independent random variables with distribution $\mathcal{N}(0,\sigma^2)$, the union-bound inequality and an application of the Gaussian tail bound in \autoref{lem:gaussian_tails} yields that for all $x \geq 0$, $P(\forall i \in \{1,\dots,d\}, \, X_i \leq x) \geq 1 - \sum_{i=1}^{d} P(X_i \geq x) \geq 1 - \frac{d \sigma}{x\sqrt{2\pi}} e^{-\frac{x^2}{2\sigma^2}}$. For this to hold with probability at least $1-\epsilon$ when $x = x_0$, we can impose
\begin{align} \label{eq:prob_bound}
    \frac{d \sigma}{x_0\sqrt{2\pi}} e^{-\frac{x^2}{2\sigma^2}} = \epsilon \iff \frac{x^2}{2\sigma^2} = \log\left(\frac{d \sigma}{\sqrt{2\pi}\epsilon x_0} \right), 
\end{align}
which is the defining equation of the sequence $(\sigma_d)_{d}$.

Suppose that $\sigma_{d+1} \geq \sigma_d$. Then, 
\begin{align}
\frac{x_0^2}{2\sigma_{d+1}^2} \leq \frac{x_0^2}{2\sigma_d^2} = \log\left(\frac{d \sigma_d}{\sqrt{2\pi}\epsilon x_0} \right) < \log\left(\frac{(d+1) \sigma_{d+1}}{\sqrt{2\pi}\epsilon x_0} \right) = \frac{x_0^2}{2\sigma_{d+1}^2},
\end{align}
which is a contradiction. Now, take the sequence $(\tilde{\sigma}_k)_k$ defined as $\tilde{\sigma}_d = C/\log(d)$ for any $C > 0$. We have that $\frac{x_0^2}{2\tilde{\sigma}_d^2} = \frac{x_0^2 \log(d)^2}{2 C^2}$ and $\log\left(\frac{d \tilde{\sigma}_d}{\sqrt{2\pi}\epsilon} \right) = \log\left(\frac{d C}{\log(d)\sqrt{2\pi}\epsilon} \right)$.
Since $\log(d/\log(d))$ is asymptotically smaller than $\log(d)^2$, there exists $d_0 \in \mathbb{Z}_{+}$ such that for all $d \geq d_0$, $\frac{x_0^2}{2\tilde{\sigma}_d^2} > \log\left(\frac{d \tilde{\sigma}_d}{\sqrt{2\pi}\epsilon x_0} \right)$, which implies that for $d \geq d_0$, we have $\sigma_d > \tilde{\sigma}_d = C/\log(d)$.
\qed

\vspace{10pt}
\noindent\textbf{Proof of \autoref{prop:Z_pm_high_prob}.}
As argued in the main text, with probability at least $1-2\epsilon$, $\sum_{i=1}^d f_1(Z^+_i) = \sum_{i=1}^d \text{sign}(\xi^+_i)$ and $\sum_{i=1}^d f_1(Z^-_i) = \sum_{i=1}^d \text{sign}(\xi^-_i)$. Since $\xi^+$ and $\xi^-$ have an even (resp. odd) number of components taking negative values, we have that 
\begin{align} \label{eq:sum_xi}
	\sum_{i=1}^d \text{sign}(\xi^+_i) \equiv 
	\begin{cases}
	0 \ (\text{mod }4) &\text{if } d \equiv 0 \ (\text{mod }4) \\
	1 \ (\text{mod }4) &\text{if } d \equiv 1 \ (\text{mod }4) \\
	2 \ (\text{mod }4) &\text{if } d \equiv 2 \ (\text{mod }4) \\
	3 \ (\text{mod }4) &\text{if } d \equiv 3 \ (\text{mod }4)
	\end{cases}, 
	\quad 
	\sum_{i=1}^d \text{sign}(\xi^-_i) \equiv 
	\begin{cases}
	2 \ (\text{mod }4) &\text{if } d \equiv 0 \ (\text{mod }4) \\
	3 \ (\text{mod }4) &\text{if } d \equiv 1 \ (\text{mod }4) \\
	0 \ (\text{mod }4) &\text{if } d \equiv 2 \ (\text{mod }4) \\
	1 \ (\text{mod }4) &\text{if } d \equiv 3 \ (\text{mod }4)
	\end{cases}
\end{align}
By the construction of $f_2$ (see \autoref{fig:f_1_2} (center, right)),
\begin{align} \label{eq:f2_cases}
	f_2(x) = 
	\begin{cases}
		1 &\text{if } x \equiv 0 \ (\text{mod }4) \\
		-1 &\text{if } x \equiv 2 \ (\text{mod }4)
	\end{cases}
	\text{ if $d$ odd,} \quad 
	f_2(x) = 
	\begin{cases}
		1 &\text{if } x \equiv 1 \ (\text{mod }4) \\
		-1 &\text{if } x \equiv 3 \ (\text{mod }4)
	\end{cases}
	\text{ if $d$ even.}
\end{align}
The equations \eqref{eq:sum_xi} together with \eqref{eq:f2_cases} show the high-probability statements for $F(Z^{+})$ and $F(Z^{-})$.
To show the lower bound, note that $F(Z^+), F(Z^-)$ are different from $1,-1$ respectively with probability at most $2\epsilon$. Since $|F|$ is upper-bounded by 1, when $d \equiv 0,1 \ (\text{mod } 4)$ have that
\begin{align}
    \mathbb{E}[F(Z^{+})] &\geq P(F(Z^{+}) = 1) - P(F(Z^{+}) \neq 1) \geq 1-2\epsilon-2\epsilon = 1-4\epsilon \\
    \mathbb{E}[F(Z^{-})] &\leq -P(F(Z^{-}) = -1) + P(F(Z^{-}) \neq -1) \leq -(1-2\epsilon)+2\epsilon = -1+4\epsilon,
\end{align}
When $d \equiv 2,3 \ (\text{mod } 4)$ the roles of $Z^{+}$ and $Z^{-}$ get reversed. This concludes the proof.
\qed

\vspace{10pt}
\noindent\textbf{Proof of \autoref{lem:path_norm_F}.}
$F$ can be expressed as a three-layer neural network because both $f_1$ and $f_2$ are two-layer networks.
The path-norm with bias of $F$ for $d$ even is:
\begin{align}
\begin{split}
    %&\left( 4 + 2 \sum_{i=1}^{(d-2)/2} 2 \right) \left(\sum_{i=1}^d \sum_{b \in \mathcal{B}} \frac{1}{x_0} \left( (1+|b-2x_0|) + (1+|b-x_0|) + (1+|b+x_0|) + (1+|b+2x_0|) \right) \right) 
    \text{PN}_b(\mathcal{W}) &= %\left( 4 + 2 \sum_{i=1}^{(d-2)/2} 2 \right) \left(\sum_{i=1}^d \sum_{\beta \in \mathcal{B}} \frac{1}{x_0} \left( \sqrt{1+(\beta-2x_0)^2} + \sqrt{1+(\beta-x_0)^2} + \sqrt{1+(\beta+x_0)^2} + \sqrt{1+(\beta+2x_0)^2} \right) \right) 
    \left( 4 + 2 \sum_{i=1}^{(d-2)/2} 2 \right) \left(\sum_{i=1}^d \sum_{\beta \in \mathcal{B}} \sum_{j=-2}^{2} \frac{\sqrt{1+(\beta+jx_0)^2}}{x_0} \right) 
    + 2d + 2 \sum_{i=1}^{(d-2)/2} (2i + 2i) + 1 \\ &= 2d \left(\sum_{i=1}^d \sum_{\beta \in \mathcal{B}} \frac{1}{x_0} \left( 4+4|\beta| \right) \right) + 2d + 8 \sum_{i=1}^{(d-2)/2} i + 1 \\ &= \frac{64 d^2}{x_0} + d(d-2) + 2d + 1 = \left( \frac{64}{x_0} + 1 \right) d^2 + 1
\end{split}
\end{align}
In the second equality we bounded $\sqrt{1+(\beta+jx_0)^2}$ by $1 + |\beta+j x_0|$, and in the third equality we used that $\sum_{\beta \in \mathcal{B}} |\beta| = |-3/2|+|-1/2|+|1/2|+|3/2| = 4$ and that $\sum_{i=1}^{(d-2)/2} i = \frac{d(d-2)}{8}$. The path-norm without bias for $d$ even is $\text{PN}_{nb}(\mathcal{W}) = \left( 4 + 2 \sum_{i=1}^{(d-2)/2} 2 \right) \frac{16d}{x_0} = \frac{32 d^2}{x_0}$.
For $d$ odd, the path-norm with bias is:
\begin{align}
\begin{split}
    \text{PN}_b(\mathcal{W}) &= \left( 4 + 2 \sum_{i=0}^{(d-3)/2} 2 \right) \left( \sum_{i=1}^d \sum_{\beta \in \mathcal{B}} \sum_{j=-2}^{2} \frac{\sqrt{1+(\beta+jx_0)^2}}{x_0} \right) +  2d + 2 \sum_{i=0}^{(d-3)/2} (2i + 1 + 2i + 1) \\ &= (2d + 2) \left(\sum_{i=1}^d \sum_{\beta \in \mathcal{B}} \frac{1}{x_0} \left( 4+4|\beta| \right) \right) + 2d + 8 \sum_{i=0}^{(d-3)/2} i + 4 (1+(d-3)/2) \\ &= \left( \frac{64}{x_0} + 1\right)d^2 + \frac{64d}{x_0} + 2
\end{split}    
\end{align}
In the third equality we used that $\sum_{i=0}^{(d-3)/2} i = \frac{(d-3)(d-1)}{8}$. The path-norm without bias for $d$ odd is $\text{PN}_{nb}(\mathcal{W}) = \left( 4 + 2 \sum_{i=0}^{(d-3)/2} 2 \right) \frac{16d}{x_0} = \frac{32 d^2 + 32 d}{x_0}$.
\qed

\vspace{-10pt}
\section{Proofs of \autoref{sec:separation_2_v_rkhs}} \label{sec:proofs_2_v_rkhs}
\begin{lemma} \label{lem:2_v_rkhs_derivation}
    The expression for $\int_{\R^d} \rho_d(x) \sigma(\langle \theta, x \rangle - b) \, dx$ in equation \eqref{eq:discriminator_fourier} can be simplified to \eqref{eq:discriminator_fourier_2}.
\end{lemma}
\begin{proof} First, note that
\begin{align}
&\frac{i}{\sqrt{2\pi}} \left(e^{-\frac{\|t\theta-\ell e_1\|^2}{2\sigma^2}} - e^{-\frac{\|t\theta+\ell e_1\|^2}{2\sigma^2}} \right) e^{-itb} = \frac{i}{\sqrt{2\pi}} \left(e^{-\frac{t^2 - 2 \ell t \theta_1 + \ell^2}{2\sigma^2}} - e^{-\frac{t^2 + 2 \ell t \theta_1 + \ell^2}{2\sigma^2}} \right) e^{-itb} \\ &= \frac{i e^{-\frac{\ell^2}{2\sigma^2}} e^{-\frac{t^2}{2\sigma^2}}}{\sqrt{2\pi}} \left(e^{\frac{\ell t \theta_1}{\sigma^2}} - e^{-\frac{\ell t \theta_1}{\sigma^2}} \right) e^{-itb} = \sqrt{\frac{2}{\pi}} i e^{-\frac{\ell^2}{2\sigma^2}} e^{-\frac{t^2}{2\sigma^2} - itb} \sinh\left( \frac{\ell t \theta_1}{\sigma^2} \right).
\end{align}
And
\begin{align}
\begin{split}
    &\int_{\R} \frac{d}{dt} \delta(t) \left( e^{-\frac{t^2}{2\sigma^2} - itb} \sinh\left( \frac{\ell t \theta_1}{\sigma^2} \right) \right) \, dt = -\frac{d}{dt} \left(  e^{-\frac{t^2}{2\sigma^2} - itb} \sinh\left( \frac{\ell t \theta_1}{\sigma^2} \right) \right) \bigg\rvert_{t=0} \\ &= \left( \left( \frac{t}{\sigma^2} + ib \right) e^{-\frac{t^2}{2\sigma^2} - itb} \sinh\left( \frac{\ell t \theta_1}{\sigma^2} \right) - \frac{\ell \theta_1}{\sigma^2} e^{-\frac{t^2}{2 \sigma^2}-itb} \cosh\left( \frac{\ell t \theta_1}{\sigma^2} \right) \right) \bigg\rvert_{t=0} = - \frac{\ell \theta_1}{\sigma^2}.
\end{split}
\end{align}
Let us set
\begin{align}
    u(t) = -\frac{d}{dt} \left( e^{-\frac{t^2}{2\sigma^2} - itb} \sinh\left( \frac{\ell t \theta_1}{\sigma^2} \right) \right) = e^{-\frac{t^2}{2\sigma^2}-itb} \left( \left( \frac{t}{\sigma^2} + ib \right) \sinh\left( \frac{\ell t \theta_1}{\sigma^2} \right) - \frac{\ell \theta_1}{\sigma^2} \cosh\left( \frac{\ell t \theta_1}{\sigma^2} \right) \right)
\end{align}
Since $u(-t) = e^{-\frac{t^2}{2\sigma^2}+itb} \left( \left( \frac{t}{\sigma^2} - ib \right) \sinh\left( \frac{\ell t \theta_1}{\sigma^2} \right) - \frac{\ell \theta_1}{\sigma^2} \cosh\left( \frac{\ell t \theta_1}{\sigma^2} \right) \right) = \overline{u(t)}$, we have that $u(t) - \overline{u(-t)} = 2 i \text{Im}(u(t))$. And 
\begin{align}
    &2\text{Im}(u(t)) = 2e^{-\frac{t^2}{2\sigma^2}}  \left( \sin(tb) \left( -\frac{t}{\sigma^2} \sinh\left(\frac{\ell t\theta_1}{\sigma^2} \right) + \frac{\ell \theta_1}{\sigma^2} \cosh \left(\frac{\ell t \theta_1}{\sigma^2}\right) \right) + b \sinh \left(\frac{\ell t \theta_1}{\sigma^2}\right) \cos(tb) \right) \\ &= e^{-\frac{t^2}{2\sigma^2}}  \left( \left( b \cos(tb) - \frac{t}{\sigma^2} \sin(tb) \right) \left( e^{\frac{\ell t\theta_1}{\sigma^2}} - e^{-\frac{\ell t\theta_1}{\sigma^2}} \right) + \frac{\ell \theta_1}{\sigma^2} \sin(tb) \left( e^{\frac{\ell t\theta_1}{\sigma^2}} + e^{-\frac{\ell t\theta_1}{\sigma^2}} \right) \right) \\ &= 
    e^{\frac{\ell^2}{2\sigma^2}}\left( b \cos(tb) + \frac{\ell \theta_1 - t}{\sigma^2} \sin(tb) \right) e^{-\frac{\|t\theta -\ell e_1\|^2}{2\sigma^2}} + \left( -b \cos(tb) + \frac{\ell \theta_1 + t}{\sigma^2} \sin(tb) \right) e^{-\frac{\|t\theta +\ell e_1\|^2}{2\sigma^2}}
\end{align}
Hence, $\text{p.v.} \left[ \frac{1}{i\pi t} \right] (u) = \frac{1}{\pi} \int_{0}^{\infty} \frac{2\text{Im}(u(t))}{t} \, dt$ is equal to
\begin{align}
\begin{split} \label{eq:pv_rewritten}
    \frac{e^{\frac{\ell^2}{2\sigma^2}}}{\pi} \int_{0}^{+\infty} \frac{\left( b \cos(tb) + \frac{\ell \theta_1 - t}{\sigma^2} \sin(tb) \right) e^{-\frac{\|t\theta -\ell e_1\|^2}{2\sigma^2}} + \left( -b \cos(tb) + \frac{\ell \theta_1 + t}{\sigma^2} \sin(tb) \right) e^{-\frac{\|t\theta +\ell e_1\|^2}{2\sigma^2}}}{t} \, dt. 
\end{split}
\end{align}
We simplify this further via integration by parts:
\begin{align}
\begin{split}
	&\int_{0}^{+\infty}  \frac{\sin(tb)}{t} \frac{\ell \theta_1 - t}{\sigma^2}  \exp \left(- \frac{\|t \theta - \ell e_1\|^2}{2\sigma^2}\right) \, dt = \int_{0}^{+\infty}  \frac{\sin(tb)}{t} \frac{d}{dt} \left( \exp \left(- \frac{\|t \theta - \ell e_1\|^2}{2\sigma^2}\right) \right) \, dt %\\ &= \left[ \frac{\sin(tb)}{t} \exp \left(- \frac{\|t \theta - \ell e_1\|^2}{2\sigma^2}\right) \right]^{\infty}_{0} - \int_{0}^{+\infty}  \frac{d}{dt} \left( \frac{\sin(tb)}{t} \right) \exp \left(- \frac{\|t \theta - \ell e_1\|^2}{2\sigma^2}\right) \, dt 
	\\ &= -b \exp \left( -\frac{\ell^2}{2\sigma^2}\right) - \int_{0}^{+\infty} \left( \frac{b \cos(tb)}{t} - \frac{\sin(bt)}{t^2} \right) \exp \left(- \frac{\|t \theta - \ell e_1\|^2}{2\sigma^2}\right) \, dt, \\
	&\int_{0}^{+\infty}  \frac{\sin(tb)}{t} \frac{\ell \theta_1 + t}{\sigma^2}  \exp \left(- \frac{\|t \theta + \ell e_1\|^2}{2\sigma^2}\right) \, dt = -\int_{0}^{+\infty}  \frac{\sin(tb)}{t} \frac{d}{dt} \left( \exp \left(- \frac{\|t \theta + \ell e_1\|^2}{2\sigma^2}\right) \right) \, dt %\\ &= -\left[ \frac{\sin(tb)}{t} \exp \left(- \frac{\|t \theta + \ell e_1\|^2}{2\sigma^2}\right) \right]^{\infty}_{0} + \int_{0}^{+\infty}  \frac{d}{dt} \left( \frac{\sin(tb)}{t} \right) \exp \left(- \frac{\|t \theta + \ell e_1\|^2}{2\sigma^2}\right) \, dt 
	\\ &= b \exp \left( -\frac{\ell^2}{2\sigma^2}\right) + \int_{0}^{+\infty} \left( \frac{b \cos(tb)}{t} - \frac{\sin(bt)}{t^2} \right) \exp \left(- \frac{\|t \theta + \ell e_1\|^2}{2\sigma^2}\right) \, dt
\end{split} 
\end{align}
Using this, equation \eqref{eq:pv_rewritten} becomes
\begin{align}
    \frac{e^{\frac{\ell^2}{2\sigma^2}}}{\pi} \int_{0}^{+\infty} \frac{ \sin(tb) \left( e^{-\frac{\|t\theta -\ell e_1\|^2}{2\sigma^2}} - e^{-\frac{\|t\theta +\ell e_1\|^2}{2\sigma^2}} \right)}{t^2} \, dt.
\end{align}
Putting everything together yields equation \eqref{eq:discriminator_fourier_2}.
\end{proof}

\begin{lemma} \label{lem:v_bound}
Letting $v(t) = \frac{\sin(tb)}{t^2} \left( \exp(-\frac{\|t\theta -\ell e_1\|^2}{2\sigma^2}) - \exp(-\frac{\|t\theta +\ell e_1\|^2}{2\sigma^2}) \right)$, we have
\begin{align} 
\begin{split} \label{eq:v_bound}
    &\bigg| \int_{0}^{\frac{2\sigma^2}{\ell \theta_1}} v(t) \, dt \bigg| \leq |b| (e + e^{-1}) e^{- \frac{\ell^2}{2\sigma^2}}, \quad \bigg| \int_{\frac{2\sigma^2}{\ell \theta_1}}^{1} v(t) \, dt \bigg| \leq \frac{\ell^2 \theta_1^2 \exp \left( - \frac{(\ell-1)^2}{2\sigma^2} \right)}{4 \sigma^4}, \\
    &\bigg| \int_{1}^{+\infty} v(t) \, dt \bigg| \leq \sqrt{2\pi \sigma^2} \exp(-\frac{\ell^2(1-\theta_1^2)}{2\sigma^2})
\end{split}
\end{align}
\end{lemma}
\begin{proof}
First, note that $v(t) = 2 e^{- \frac{t^2 + \ell^2}{2\sigma^2}} \frac{\sin(tb)}{t^2} \sinh \left( \frac{\ell t \theta_1}{2 \sigma^2} \right)$. Then,
\begin{align}
\begin{split} \label{eq:bound_B}
    &\bigg| \int_{0}^{\frac{2\sigma^2}{\ell \theta_1}} v(t) \, dt \bigg| = 2 \bigg| \int_{0}^{\frac{2\sigma^2}{\ell \theta_1}} e^{- \frac{t^2 + \ell^2}{2\sigma^2}} \frac{\sin(tb)}{t^2} \sinh \left( \frac{\ell t \theta_1}{2 \sigma^2} \right) \, dt \bigg| \\ &\leq 2 \int_{0}^{\frac{2\sigma^2}{\ell \theta_1}} e^{- \frac{\ell^2}{2\sigma^2}} \frac{(e + e^{-1}) \ell \theta_1 |b|}{4\sigma^2} \, dt \leq |b| (e + e^{-1}) e^{- \frac{\ell^2}{2\sigma^2}} 
\end{split}
\end{align}
Here, we used that $e^{- \frac{t^2 + \ell^2}{2\sigma^2}} \leq e^{- \frac{\ell^2}{2\sigma^2}}$ and that by the mean value theorem,
\begin{align} 
\begin{split}
    \forall t \in \left[0,\frac{2\sigma^2}{\ell \theta_1} \right], \quad &\bigg|\frac{\sin(tb)}{t} \bigg| = |b \cos(b\tilde{t})| \leq |b|, \text{ and } \bigg|\frac{\sinh \left( \frac{\ell t \theta_1}{2 \sigma^2} \right)}{t} \bigg| = \bigg|\frac{\ell \theta_1 \cosh \left( \frac{\ell \tilde{t} \theta_1}{2 \sigma^2} \right)}{2 \sigma^2} \bigg| \leq \frac{(e+e^{-1})\ell \theta_1}{4 \sigma^2}. \\
\end{split}
\end{align}
The second inequality in \eqref{eq:v_bound} holds because:
\begin{align}
\begin{split} \label{eq:bound_C}
    \bigg| \int_{\frac{2\sigma^2}{\ell \theta_1}}^{1} \frac{ \sin(tb) \left( e^{-\frac{\|t\theta -\ell e_1\|^2}{2\sigma^2}} - e^{-\frac{\|t\theta +\ell e_1\|^2}{2\sigma^2}} \right)}{t^2} \, dt \bigg| \leq %\frac{\exp \left( - \frac{(\ell-1)^2}{2\sigma^2} \right)}{\left( \frac{2\sigma^2}{\ell \theta_1} \right)^2} =
    \frac{\ell^2 \theta_1^2 e^{- \frac{(\ell-1)^2}{2\sigma^2}}}{4 \sigma^4},
\end{split}    
\end{align}
where we used that for any $t \in [\frac{2\sigma^2}{\ell \theta_1},1]$, $\|t\theta \pm \ell e_1\|^2 = t^2 \pm 2 \ell \theta_1 t + \ell^2 \geq t^2 - 2 \ell + \ell^2 = (\ell - 1)^2$. Now, without loss of generality, suppose that $\theta_1 > 0$. Then,
\begin{align}
    &\left| \int_{1}^{+\infty} \frac{ \sin(tb) \left(e^{-\frac{\|t\theta -\ell e_1\|^2}{2\sigma^2}} - e^{-\frac{\|t\theta +\ell e_1\|^2}{2\sigma^2}} \right)}{t^2} \, dt \right| \leq \int_{1}^{+\infty} e^{-\frac{t^2 - 2 \ell \theta_1 t + \ell^2}{2\sigma^2}} \, dt \\ &= \int_{1}^{+\infty} e^{-\frac{(t - \ell \theta_1)^2 + \ell^2(1-\theta_1^2)}{2\sigma^2}} \, dt \leq \sqrt{2\pi \sigma^2} e^{-\frac{\ell^2(1-\theta_1^2)}{2\sigma^2}}
\end{align}
The same bound is obtained if $\theta_1 < 0$ and this shows the third inequality in \eqref{eq:v_bound}.
\end{proof}

\begin{lemma}[\cite{li2011concise}] \label{lem:li2011}
Let $\theta \in (0,\pi/2]$ and consider the $(d-1)$-spherical cap with colatitude angle $\theta$, i.e. $C_{r,\theta} = \{ x \in \R^d \ | \ \|x\| = r, \ \langle x, e_1 \rangle \geq \cos(\theta) \}$. The area of $C_{r,\theta}$ is $A_{r,\theta} = \frac{2\pi^{(d-1)/2}}{\Gamma(\frac{d-1}{2})} r^{d-1} \int_{0}^{\theta} \sin^{d-2} (t) \, dt = \text{vol}(\mathbb{S}^{d-1}) \frac{\Gamma(\frac{d}{2})}{\Gamma(\frac{d-1}{2}) \Gamma(\frac{1}{2})} r^{d-1} \int_{0}^{\theta} \sin^{d-2} (t) \, dt$.
\end{lemma}

\vspace{10pt}
\noindent\textbf{Proof of \autoref{prop:upper_bound_rkhs}.}
Using the bounds from \autoref{lem:v_bound}, we have that $|\int_{\R^d} \rho_d(x) \sigma(\langle \theta, x \rangle - b) \, dx| = \sqrt{\frac{2}{\pi}} |\frac{-A \ell \theta_1}{\sigma^2} e^{-\frac{\ell^2}{2\sigma^2}} + B\int_{0}^{+\infty} v(t) \, dt|$ is upper-bounded by
\begin{align} \label{eq:upper_bound_theta_b}
     \sqrt{\frac{2}{\pi}} \left( \frac{|A| \ell \theta_1}{\sigma^2} e^{-\frac{\ell^2}{2\sigma^2}} + |B|\left(\sqrt{2\pi \sigma^2} e^{-\frac{\ell^2(1-\theta_1^2)}{2\sigma^2}} + |b| (e + e^{-1}) e^{- \frac{\ell^2}{2\sigma^2}} + \frac{\ell^2 \theta_1^2 e^{- \frac{(\ell-1)^2}{2\sigma^2}}}{4 \sigma^4} \right) \right).
\end{align}
To keep things simple, we use a crude upper bound on the square of \eqref{eq:upper_bound_theta_b} via the rearrangement inequality and we integrate with respect to $\tau(\theta,b)$:
\begin{align}
\begin{split} \label{eq:square_rearrangement}
    &\frac{2}{\pi} \int_{\mathbb{S}^{d-1} \times \R} \left( \frac{4 |A|^2 \ell^2 \theta_1^2}{\sigma^4} e^{-\frac{\ell^2}{\sigma^2}} + |B|^2\left(8\pi \sigma^2 e^{-\frac{\ell^2(1-\theta_1^2)}{\sigma^2}} + 4b^2 (e + e^{-1})^2 e^{- \frac{\ell^2}{\sigma^2}} + \frac{4 \ell^4 \theta_1^4 e^{- \frac{(\ell-1)^2}{\sigma^2}}}{16 \sigma^8} \right) \right) \, d\tau(\theta,b) \\ &= \frac{2|B|^2}{\pi} \bigg( 8\pi \sigma^2 \int_{\mathbb{S}^{d-1}} e^{-\frac{\ell^2(1-\theta_1^2)}{\sigma^2}} \, d\tau(\theta) + \frac{4(e + e^{-1})^2 e^{- \frac{\ell^2}{\sigma^2}}}{\sqrt{2\pi}} \int_{\R} b^2 e^{-b^2/2} \, db + \frac{4 \ell^4 \int_{\mathbb{S}^{d-1}} \theta_1^4 \, d\tau(\theta) \, e^{- \frac{(\ell-1)^2}{\sigma^2}}}{16 \sigma^8} \bigg) \\ &+ \frac{8 |A|^2 \ell^2 e^{-\frac{\ell^2}{\sigma^2}}}{\pi \sigma^4} \int_{\mathbb{S}^{d-1}} \theta_1^2 \, d\tau(\theta).
\end{split}
\end{align}
Here, we use $\tau$ to denote the uniform probability over $\mathbb{S}^{d-1}$ as well. By \autoref{lem:li2011}, we have that
\begin{align}
\begin{split} \label{eq:area_cap_application}
    &\int_{\mathbb{S}^{d-1}} \exp(-\frac{\ell^2(1-\theta_1^2)}{\sigma^2}) \, d\tau(\theta) = \frac{1}{\text{vol}(\mathbb{S}^{d-1})} \int_{0}^{\pi/2} e^{-\frac{\ell^2(1-\cos^2(t))}{\sigma^2}} \frac{dA_{1,t}}{dt}(t) \, dt \\ &= \frac{\Gamma(\frac{d}{2})}{\Gamma(\frac{d-1}{2}) \Gamma(\frac{1}{2})} \int_{0}^{\pi/2} e^{-\frac{\ell^2(1-\cos^2(t))}{\sigma^2}} \sin^{d-2} (t) \, dt \\ &\leq \frac{\Gamma(\frac{d}{2})}{\Gamma(\frac{d-1}{2}) \Gamma(\frac{1}{2})} \left(\int_{0}^{\pi/4} e^{-\frac{\ell^2(1-\cos^2(t))}{\sigma^2}} \sin^{d-2} (t) \, dt + \int_{\pi/4}^{\pi/2} e^{-\frac{\ell^2(1-\cos^2(t))}{\sigma^2}} \sin^{d-2} (t) \, dt \right)
    \\ &\leq \frac{\frac{\pi}{4} \Gamma(\frac{d}{2})}{\Gamma(\frac{d-1}{2}) \Gamma(\frac{1}{2})} \left( \frac{1}{2^{(d-2)/2}} +  e^{-\frac{\ell^2}{2\sigma^2}} \right)
\end{split}
\end{align}
Note that $\Gamma(1/2) = \sqrt{\pi}$, and by Stirling's approximation, $\log \Gamma(z) \leq z \log(z) - z +\frac{1}{2} \log(\frac{2 \pi}{z})$, which means that 
\begin{align}
\begin{split}
    &\log \frac{\Gamma(\frac{d}{2})}{\Gamma(\frac{d-1}{2})} \sim %\frac{d}{2} \log\left(\frac{d}{2} \right) - \frac{d}{2} +\frac{1}{2} \log \left(\frac{4 \pi}{d} \right) - \left( \frac{d-1}{2} \log \left(\frac{d-1}{2} \right) - \frac{d-1}{2} +\frac{1}{2} \log \left(\frac{4 \pi}{d-1} \right) \right) \\ &= 
    \frac{1}{2} \log \left( \frac{d}{2} \right) + \frac{d-1}{2} \log \left( 1 + \frac{1}{d-1} \right) + \frac{1}{2} + \frac{1}{2} \log\left(1 - \frac{1}{d} \right) \sim \frac{1}{2} \log \left( \frac{d}{2} \right) + 1 - \frac{1}{2d}.
\end{split}
\end{align}
Thus, the right-hand side of \eqref{eq:area_cap_application} admits the upper bound $O ( \sqrt{d} ( \frac{1}{2^{d/2}} +  e^{-\frac{\ell^2}{2\sigma^2}}))$. The other terms in the right-hand side of \eqref{eq:square_rearrangement} can be bound trivially. We take the square root and use that the square root of a sum is less or equal than the sum of square roots, which yields equation \eqref{eq:upper_bound_rkhs_prop}.
\qed

\vspace{10pt}
\noindent\textbf{Proof of \autoref{prop:lower_bound_2l}.}
Let us set $\theta = e_1$ and $b\in \R$ such that $\sin(b \ell) = 1$, which is equivalent to $b \ell = 2\pi k + \pi/2$ for some $k \in \mathbb{Z}$. Take $x_0 > 0$ such that $b x_0 \leq \pi/4$, and $0 < \epsilon < 1$ fixed. We take $\sigma$ such that $\frac{x_0^2}{2\sigma^2} = \log \left( \frac{\sqrt{2} d^2 \sigma}{\sqrt{\pi} x_0} \right)$.
With probability at least $1-\epsilon$, a Gaussian random variable  $X \sim \mathcal{N}(\ell,\sigma_0^2)$ is in $[\ell - x_0, \ell + x_0]$. Thus,
\begin{align} \label{eq:lower_bound_A}
    \int_{\ell - x_0}^{\ell + x_0} \frac{\sin(tb) e^{-\frac{\|t\theta -\ell e_1\|^2}{2\sigma^2}}}{t^2} \, dt \geq (1-\epsilon)\frac{\sqrt{2\pi \sigma^2} \sin(\frac{\pi}{4})}{(\ell + x_0)^2} = (1-\epsilon)\frac{\sqrt{\pi \sigma^2}}{(\ell + x_0)^2}.
\end{align}
Moreover,
\begin{align} \label{eq:upper_bound_A}
    \int_{\ell - x_0}^{\ell + x_0} \frac{\sin(tb) e^{-\frac{\|t\theta +\ell e_1\|^2}{2\sigma^2}}}{t^2} \, dt \leq \frac{e^{- \frac{(2\ell - x_0)^2}{2\sigma^2}}}{(\ell - x_0)^2}.
\end{align}
Also, if we take $v(t)$ as in \autoref{lem:v_bound},
\begin{align}
\begin{split} \label{eq:bound_D}
    \left| \int_{[1,+\infty] \setminus [\ell - x_0, \ell+x_0]} v(t) \, dt \right| \leq \int_{[1,+\infty] \setminus [\ell - x_0, \ell+x_0]} \frac{ |\sin(tb)| e^{-\frac{\|t\theta -\ell e_1\|^2}{2\sigma^2}}}{t^2} \, dt \leq \epsilon
\end{split}    
\end{align}
Putting together \eqref{eq:lower_bound_A}, \eqref{eq:upper_bound_A}, \eqref{eq:bound_D}, and the first two inequalities in \eqref{eq:v_bound}, we obtain that $|\int_{\R^d} \rho_d(x) \sigma(\langle \theta, x \rangle - b) \, dx|$ is lower-bounded by 
\begin{align}
\begin{split} \label{eq:lower_bound_2l_intermediate}
    &\sqrt{\frac{2}{\pi}}|B|\bigg((1-\epsilon)\frac{\sqrt{\pi \sigma^2}}{(\ell + x_0)^2} - \frac{e^{- \frac{(2\ell - x_0)^2}{2\sigma^2}}}{(\ell - x_0)^2} - |b| (e + e^{-1}) e^{- \frac{\ell^2}{2\sigma^2}} - \frac{\ell^2 \theta_1^2 e^{- \frac{(\ell-1)^2}{2\sigma^2}}}{4 \sigma^4} - \epsilon \bigg) \\
    &- \sqrt{\frac{2}{\pi}} \frac{|A| \ell}{\sigma^2} e^{-\frac{\ell^2}{2\sigma^2}}
\end{split}
\end{align}
Taking $\ell = \sqrt{d}$, $x_0 = 1$ and $\epsilon = 1/d^2$, we can set $b = \frac{\pi}{2\ell} = \frac{\pi}{2\sqrt{d}}$, which is smaller or equal than $\pi/4$ for $d \geq 4$. By the argument of \autoref{lem:sigma_d}, $\sigma_d \geq K/\log(d)$ for some constant $K$. The only asymptotically relevant terms of \eqref{eq:lower_bound_2l_intermediate} are the two involving $\epsilon$, which are the only ones not decreasing exponentially in $d$. Thus, we lower-bound 
\begin{align}
\begin{split}
    \sqrt{\frac{2}{\pi}} \frac{|B| K \sqrt{\pi} (1-1/d^2)}{\log(d) (\sqrt{d}+1)^2} - O(1/d^2) = \Omega \left( \frac{1}{d \log(d)}\right) - O(1/d^2) = \Omega \left( \frac{1}{d \log(d)}\right)
\end{split}
\end{align}
The only statement left to prove is the upper bound $\sigma_d \leq 2$, which by the monotonicity of $(\sigma_d)$ follows from the upper bound on $\sigma_0$. We have $\sigma_0 \leq 2$ because $\frac{1}{2\cdot 2^2} = \frac{1}{8}$ is smaller than $\log \left( \frac{2\sqrt{2}}{\sqrt{\pi}} \right) = \frac{1}{2} \log \left( \frac{8}{\pi} \right) = 0.467\dots$; the two curves must intersect at a value of $\sigma$ smaller than 2.
\qed

\end{document}